\pdfoutput=1
\documentclass{article}

\usepackage[utf8]{inputenc} 
\usepackage[T1]{fontenc}    

\usepackage{amsmath} 
\usepackage{amsfonts} 
\usepackage[scale=1.0]{XCharter}

\usepackage[
  paper  = letterpaper,
  top    = 1.0in,
  bottom = 1.0in,
  ]{geometry}

\usepackage[usenames,dvipsnames,table]{xcolor}
\definecolor{shadecolor}{gray}{0.9}

\usepackage[final,expansion=alltext]{microtype}
\usepackage[english]{babel}
\usepackage[parfill]{parskip}
\usepackage{afterpage}
\usepackage{framed}

{\endMakeFramed}

\usepackage{lineno}

\usepackage{ragged2e}

\newcounter{parcount}

\usepackage{graphicx}
\usepackage{wrapfig}
\usepackage[labelfont=bf,font=footnotesize,width=.9\textwidth]{caption}
\usepackage[format=hang]{subcaption}

\usepackage{booktabs,multirow,multicol}       

\usepackage[algoruled,algo2e]{algorithm2e}
\usepackage{listings}
\usepackage{fancyvrb}
\fvset{fontsize=\normalsize}

\usepackage[colorlinks,linktoc=all]{hyperref}
\usepackage[all]{hypcap}
\hypersetup{citecolor=MidnightBlue}
\hypersetup{linkcolor=MidnightBlue}
\hypersetup{urlcolor=MidnightBlue}

\usepackage[nameinlink, capitalize]{cleveref}

\usepackage[acronym,smallcaps,nowarn]{glossaries}[=v4.46]

\newcommand{\red}[1]{\textcolor{BrickRed}{#1}}

\newcommand{\blue}[1]{\textcolor{MidnightBlue}{#1}}

\lstdefinestyle{mystyle}{
    commentstyle=\color{OliveGreen},
    keywordstyle=\color{BurntOrange},
    numberstyle=\tiny\color{black!60},
    stringstyle=\color{MidnightBlue},
    basicstyle=\ttfamily,
    breakatwhitespace=false,
    breaklines=true,
    captionpos=b,
    keepspaces=true,
    numbers=left,
    numbersep=5pt,
    showspaces=false,
    showstringspaces=false,
    showtabs=false,
    tabsize=2
}
\usepackage{amsthm}
\usepackage{amssymb}

\usepackage{centernot}
\usepackage{nicefrac}       
\usepackage{mathtools}
\usepackage{amsbsy}
\usepackage{amstext}
\usepackage{thmtools}
\usepackage{thm-restate}

\begingroup
    \makeatletter
    \@for\theoremstyle:=definition,remark,plain\do{%
        \expandafter\g@addto@macro\csname th@\theoremstyle\endcsname{%
            \addtolength\thm@preskip\parskip
            }%
        }
\endgroup

\renewcommand{\mid}{~\vert~}

\usepackage{bm}

\usepackage{booktabs,arydshln}
\makeatletter
\def\adl@drawiv#1#2#3{%
        \hskip.5\tabcolsep
        \xleaders#3{#2.5\@tempdimb #1{1}#2.5\@tempdimb}%
                #2\z@ plus1fil minus1fil\relax
        \hskip.5\tabcolsep}
\newcommand{\cdashlinelr}[1]{%
  \noalign{\vskip\aboverulesep
           \global\let\@dashdrawstore\adl@draw
           \global\let\adl@draw\adl@drawiv}
  \cdashline{#1}
  \noalign{\global\let\adl@draw\@dashdrawstore
           \vskip\belowrulesep}}
\makeatother

\renewcommand{\epsilon}{\varepsilon}

\declaretheorem[style=plain,name=Theorem]{theorem}

\declaretheorem[style=plain,sibling=theorem,name=Proposition]{proposition}

\declaretheorem[style=definition,name=Assumption]{assumption}
\declaretheorem[style=definition,sibling=theorem,name=Example]{example}
\declaretheorem[style=remark,sibling=theorem,name=Remark]{remark}
\declaretheorem[style=definition,name=Condition]{condition}

\newenvironment{example*}
 {\pushQED{\qed}\example}
 {\popQED\endexample}
\numberwithin{equation}{section}

\global\long\def\floor#1{\lfloor#1\rfloor}

\DeclareMathOperator*{\argmin}{argmin}

\DeclarePairedDelimiterX\Set[1]{\lbrace}{\rbrace}%
{  #1 }

\usepackage[utf8]{inputenc} 
\usepackage[T1]{fontenc}    
\usepackage{hyperref}       
\usepackage{url}            
\usepackage{booktabs}       
\usepackage{amsfonts}       
\usepackage{nicefrac}       
\usepackage{microtype}      
\usepackage{xcolor}         
\usepackage{microtype}
\usepackage{graphicx}
\usepackage{subcaption}
\usepackage{booktabs} 

\usepackage{amsmath}
\usepackage{amssymb}
\usepackage{mathtools}
\usepackage{amsthm}
\usepackage[normalem]{ulem}
\usepackage{wrapfig,lipsum,booktabs}
\usepackage{booktabs}
\usepackage{multirow}
\usepackage{minimal_macros}

\usepackage[utf8]{inputenc} 
\usepackage[T1]{fontenc}    
\usepackage{hyperref}       
\usepackage{url}            
\usepackage{booktabs}       
\usepackage{amsfonts}       
\usepackage{nicefrac}       
\usepackage{microtype}      
\usepackage[table]{xcolor}  
\usepackage{wrapfig,lipsum,booktabs}
\usepackage{multirow}
\usepackage{subcaption}
\usepackage{algorithm}
\usepackage{algpseudocode}

\definecolor{mydarkblue}{rgb}{0,0.08,0.45}

\hypersetup{ %
    pdftitle={},
    pdfauthor={},
    pdfsubject={},
    pdfkeywords={},
    pdfborder=0 0 0,
    pdfpagemode=UseNone,
    colorlinks=true,
    linkcolor=mydarkblue,
    citecolor=mydarkblue,
    filecolor=mydarkblue,
    urlcolor=mydarkblue,
    pdfview=FitH
}

\usepackage{tcolorbox}
\definecolor{linen}{RGB}{250,240,230}
\usepackage{enumitem}
\usepackage{fontawesome5}
\usepackage{caption}

\crefname{appsec}{appendix}{appendices}
\Crefname{appsec}{Appendix}{Appendices}

\usepackage{appendix}
\usepackage{array}
\newcolumntype{x}[1]{>{\centering\arraybackslash}p{#1pt}}
\newcolumntype{y}[1]{>{\raggedright\arraybackslash}p{#1pt}}

\usepackage{csquotes}
\usepackage[%
minnames=1,maxnames=99,maxcitenames=2,
style=alphabetic,
doi=false,
url=false,
giveninits=true,
hyperref,
natbib,
backend=bibtex,
sorting=nyt,
backref=true
]{biblatex}%
\renewbibmacro{in:}{%
  \ifentrytype{article}{}{\printtext{\bibstring{in}\intitlepunct}}}

\renewbibmacro*{journal}{%
  \iffieldundef{journaltitle}
    {}
    {\printtext[journaltitle]{%
       \printfield[noformat]{journaltitle}%
       \setunit{\subtitlepunct}%
       \printfield[noformat]{journalsubtitle}}}}

\DeclareFieldFormat{sentencecase}{\MakeSentenceCase*{#1}}

\renewbibmacro*{title}{%
  \ifthenelse{\iffieldundef{title}\AND\iffieldundef{subtitle}}
    {}
    {\ifthenelse{\ifentrytype{article}\OR\ifentrytype{inbook}%
      \OR\ifentrytype{incollection}\OR\ifentrytype{inproceedings}%
      \OR\ifentrytype{inreference}\OR\ifentrytype{misc}}
      {\printtext[title]{%
        \printfield[sentencecase]{title}%
        \setunit{\subtitlepunct}%
        \printfield[sentencecase]{subtitle}}}%
      {\printtext[title]{%
        \printfield[titlecase]{title}%
        \setunit{\subtitlepunct}%
        \printfield[titlecase]{subtitle}}}%
     \newunit}%
  \printfield{titleaddon}}

\AtEveryBibitem{%
\ifentrytype{article}{
    \clearfield{urldate}%
    \clearfield{eprint}
    \clearfield{eid}
}{}
\ifentrytype{book}{
    \clearfield{url}%
    \clearfield{urldate}%
    \clearfield{eprint}
}{}
\ifentrytype{collection}{
    \clearfield{url}%
    \clearfield{urldate}%
    \clearfield{eprint}
}{}
\ifentrytype{incollection}{
    \clearfield{url}%
    \clearfield{urldate}%
    \clearfield{eprint}
}{}
}

\AtEveryBibitem{
    \clearfield{pages}
    \clearfield{review}%
    \clearfield{series}
    \clearfield{volume}
    \clearfield{month}
    \clearfield{isbn}
    \clearfield{issn}
    \clearlist{location}
    \clearfield{series}
    \clearlist{publisher}
    \clearname{editor}
}{}

\crefformat{figure}{Figure~#2#1#3}

\usepackage{enumitem} 
\usepackage[separate-uncertainty=true,multi-part-units=single]{siunitx} 

\declaretheoremstyle[
spacebelow=\parsep,
    spaceabove=\parsep,
  mdframed={
    backgroundcolor=gray!10!white,     
    hidealllines=true, 
    innertopmargin=8pt, 
    innerbottommargin=4pt, 
    skipabove=8pt,
    skipbelow=10pt,
    nobreak=true
}
]{grayboxed}

\crefname{gassumption}{Assumption}{Assumptions}

\usepackage{xcolor}

\definecolor{WowColor}{rgb}{.75,0,.75}
\definecolor{SubtleColor}{rgb}{0,0,.50}

\newcounter{margincounter}

\usepackage[affil-it]{authblk}

\usepackage{makecell}

\usepackage{multirow}
\usepackage{tikz}
\usetikzlibrary{positioning}
\usetikzlibrary{arrows, automata}
\usetikzlibrary{patterns}

\def\showauthornotes{1}
\ifnum\showauthornotes=1
\newcommand{\Authornote}[2]{{\sf\small\color{blue}{[#1: #2]}}}
\else
\newcommand{\Authornote}[2]{}
\fi

\title{A Unified Rate-Distortion Perspective\\on Vector, Product, and Scalar Quantization}

\author{
  {\bfseries\upshape
  Xianghong Fang$^{1}$
  \qquad
  Wenlong Mou$^{1}$
  \qquad
  Yuan Yuan$^{2}$
  }
  \\ \vspace{-5pt}
  {\bfseries\upshape
  Dehan Kong$^{1}$
  \qquad
  Tim G. J. Rudner$^{1,3}$
  } \\ \vspace{-0pt}
  {\upshape
  \vspace*{10pt}
  $^1$University of Toronto \quad
  $^2$Boston College\quad
  $^3$Vijil
  } \\
  {\upshape
  \vspace*{15pt}
    \begin{center}
    \href{https://vq-research.github.io/Rate-Distortion-Perspective/}{\faGlobe\enspace Website}
    \quad
    \href{https://github.com/VQ-Research/Rate-Distortion-Perspective}{\faGithub\enspace Code \& Models}
    \end{center}
  }
}
\date{}

\begin{document}

\maketitle

\begin{abstract}
\noindent

Discrete visual tokenization, predominantly driven by vector, scalar, and product quantization, lacks a unified conceptual framework for understanding quantization tradeoffs. In this paper, we propose a unified rate--distortion perspective on modern discrete visual tokenization. By viewing quantization as lossy compression, we characterize the nominal fixed-length coding rate through token count and codebook size, and quantization error as the distortion. Within this framework, we resolve three central questions. First, we theoretically and empirically show that minimizing distortion, rather than maximizing codebook utilization, is the primary intrinsic objective for reconstruction fidelity, with a direct connection to the STE-induced gradient discrepancy. Second, we establish two critical fairness conditions for intrinsic quantization comparison: controlling latent feature statistics and enforcing identical coding rates. Third, under these conditions, we recover the VQ--PQ--SQ distortion hierarchy in modern visual tokenization and show empirically that modern VQ methods achieve the lowest distortion. This work provides a foundational rate--distortion reframing of modern discrete visual tokenization, resolves ambiguities in quantizer evaluation, and provides a controlled framework for isolating intrinsic quantization effectiveness under fixed-rate constraints.
\end{abstract}

\vspace*{-3pt} 
\section{Introduction} 
\label{sec:introduction} 
\vspace*{-3pt}

Discrete visual tokenization has achieved remarkable success in recent years, driven by advances in vector quantization \citep[VQ;][]{Esser2020TamingTF,Sun2024AutoregressiveMB,Tian2024VisualAM}, product quantization \citep[PQ;][]{Jgou2011ProductQF,Guo2024AddressingIC,Li2024ImageFolderAI}, and scalar quantization \citep[SQ;][]{Mentzer2023FiniteSQ,yu2024language,zhao2024image,Han2024InfinitySB}. A major line of VQ research has focused on improving codebook utilization through improved codebook updates, optimization, or parameterizations \citep{Lee2022AutoregressiveIG,Zheng2023OnlineCC,Zhu2024ScalingTC,Zhu2024AddressingRC,Shi2024TamingScalable,Chang2025FullVQ,Lu2026BeyondStationarity}. Despite this rapid progress, the conceptual relationship between utilization, distortion, and quantization effectiveness remains comparatively less understood. In particular, high or even complete codebook utilization does not necessarily imply faithful representation of the input. This motivates a more fundamental framework for understanding the objectives and tradeoffs underlying different quantization algorithms. 

In this paper, we propose a unified rate-distortion perspective on quantization to systematically address these considerations.
Rate-distortion theory provides the classical framework for characterizing lossy compression, and has also motivated learned image compression methods that jointly optimize coding rate and reconstruction distortion \citep{Balle2017EndToEnd,Lei2024ApproachingRD,Zhang2024OptimalLattice,Jiang2026RDVQ}.
We use rate-distortion theory as an analytical framework for comparing discrete visual quantizers by defining distortion as quantization error and rate as the nominal fixed-length coding budget determined by token count and codebook size.
Under this perspective, we resolve three central open questions in discrete visual tokenization, which are summarized in \Cref{fig:introduction}.

First, we investigate the primary optimization objective for quantization algorithms in modern visual tokenization. From a rate-distortion perspective, minimizing distortion rather than directly maximizing codebook utilization is the key to achieving training stability and high-fidelity reconstruction. We provide both theoretical and empirical evidence supporting distortion minimization as the more fundamental optimization principle. Specifically, we prove that minimizing distortion necessarily implies full codebook utilization under mild conditions, whereas the converse does not hold. We further establish a direct connection between distortion and the gradient discrepancy induced by the straight-through estimator(STE)~\citep{Bengio2013EstimatingOP}. Empirically, we demonstrate substantially stronger correlations between distortion and reconstruction fidelity (measured by rFID) than those observed with codebook utilization. These results motivate a distortion-centric perspective on VQ optimization beyond codebook utilization, as summarized in \Cref{fig:introduction}(a), complementing prior analyses of STE bias~\citep{Huh2023StraighteningOT,Fifty2024RestructuringVQ} and \mbox{recent efforts to address codebook collapse~\citep{Zhu2024AddressingRC,Shi2024TamingScalable,Chang2025FullVQ,Lu2026BeyondStationarity}.}

\begin{figure}[t]
    \vspace*{-6pt}
    \centering
    \includegraphics[width=0.97\linewidth]{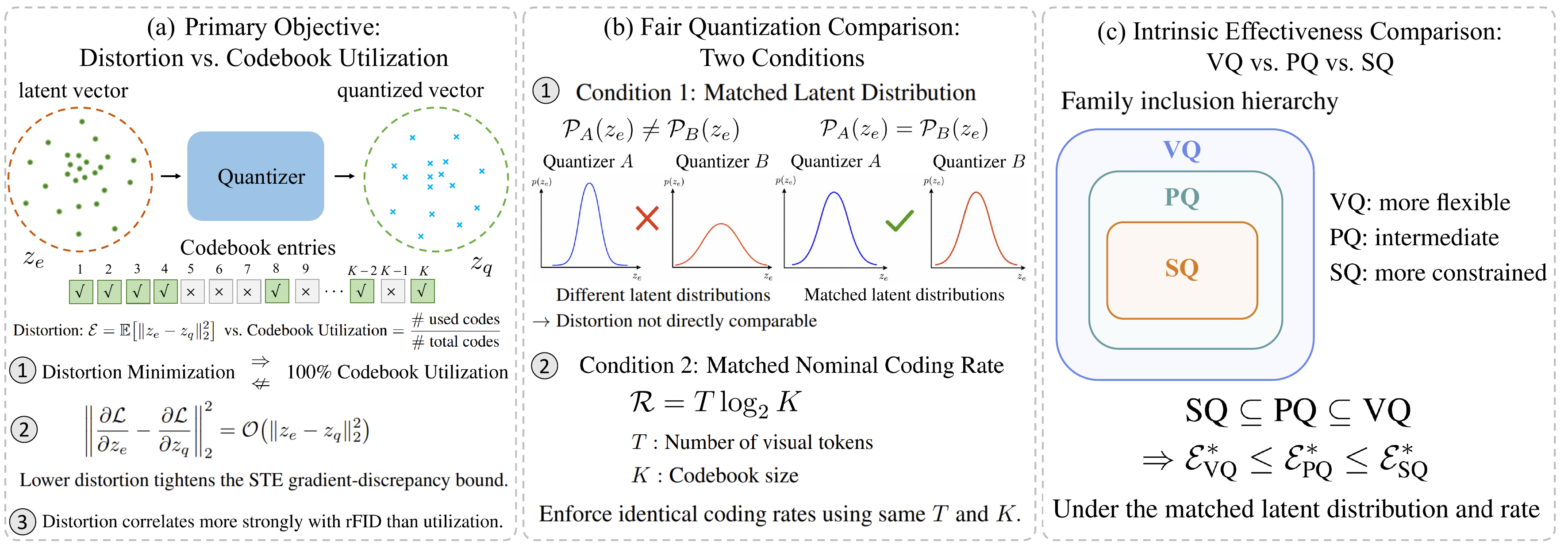}
    \vspace{-8pt}
    \caption{\small{\textbf{Overview of our unified rate-distortion perspective on quantization.}
    (a) \textbf{Primary objective.} Distortion is a more fundamental criterion than codebook utilization, with direct links to utilization, STE-induced gradient discrepancy, and reconstruction fidelity.
    (b) \textbf{Fair comparison.} Intrinsic comparison requires matched latent feature distributions and the same nominal fixed-length rate.
    (c) \textbf{Intrinsic effectiveness.} VQ, PQ, and SQ satisfy $\text{SQ}\subseteq\text{PQ}\subseteq\text{VQ}$, implying $\mathcal{E}^*_{\text{VQ}}\leq\mathcal{E}^*_{\text{PQ}}\leq\mathcal{E}^*_{\text{SQ}}$ under matched latent distributions and fixed rate.
    }}
    \label{fig:introduction}
    \vspace{-2ex}
\end{figure}

Second, we establish two fundamental conditions for fairly comparing the intrinsic effectiveness of quantization algorithms, as illustrated in \Cref{fig:introduction}(b). First, the latent feature distribution must be controlled across compared algorithms. In particular, under global rescaling, the minimum achievable squared quantization distortion changes proportionally with the latent variance, so differences in feature statistics can confound comparisons of quantizer effectiveness. Second, all algorithms must operate at the same nominal fixed-length coding rate, $R=T\log_2K$, which is jointly determined by the number of visual tokens $T$ and the cardinality of the discrete code space $K$. Only under matched latent distributions and coding rates can the intrinsic effectiveness
of different quantization algorithms be meaningfully compared.

Third, we examine the intrinsic effectiveness of quantization algorithms under rigorously controlled conditions. As illustrated in \Cref{fig:introduction}(c), SQ, PQ, and VQ form a nested hierarchy in which VQ generalizes PQ and PQ generalizes SQ, implying greater modeling flexibility and no higher optimal distortion for VQ. Related work on neural image compression has similarly shown that vector or lattice quantization can achieve lower distortion than scalar quantization when when exploiting cross-dimensional dependencies \citep{Lei2024ApproachingRD,Zhang2024OptimalLattice}. However, early visual-tokenizer studies reported severe codebook collapse for VQ, particularly with large codebooks \citep{Dhariwal2020JukeboxAG,Takida2022SQVAEVB,Yu2021VectorquantizedIM,Lee2022AutoregressiveIG,Zheng2023OnlineCC}, leading to inferior empirical performance relative to SQ and PQ in prior works \citep{Mentzer2023FiniteSQ,yu2024language,zhao2024image,Guo2024AddressingIC}. To resolve this apparent conflict, we examine whether the classical distortion hierarchy among VQ, PQ, and SQ can be realized in modern visual tokenization under controlled fixed-rate conditions. Beyond this nested ordering, we show that VQ can exploit low-dimensional source structure to achieve strictly better distortion scaling than PQ and SQ. Our empirical study further shows that advanced VQ algorithms \citep{Fang2026distributional,Fang2026VQTransplant} consistently achieve lower distortion and better reconstruction under the two fairness conditions above.

Our key contributions are as follows:
\vspace*{-3pt} \begin{enumerate}[topsep=0pt, align=left, leftmargin=15pt, labelindent=1pt, listparindent=\parindent, labelwidth=0pt, itemindent=!, itemsep=3pt, parsep=0pt] 
\item \textbf{We establish distortion minimization as a fundamental intrinsic optimization principle for modern visual quantization.} We show that distortion minimization implies full codebook utilization under mild conditions, whereas the converse does not hold. We further connect quantization distortion to the gradient discrepancy induced by the straight-through estimator..
\item \textbf{We identify two fundamental conditions for fair intrinsic comparison of quantization algorithms.} We show that comparisons must control both the scale of latent features and the coding rate, since differences in either factor can obscure the intrinsic effectiveness of the quantizer itself.
\item We characterize the VQ--PQ--SQ distortion hierarchy in modern visual tokenization and show that \textbf{VQ can exploit low-dimensional source structure to achieve strictly better distortion scaling than PQ and SQ.} Empirically, modern VQ methods achieve the lowest distortion and best reconstruction after decoder adaptation under controlled comparisons.
\end{enumerate} 
\vspace*{-1.5pt}
\section{Background}
\label{sec:background}
\vspace*{-1.5pt}

\subsection{Discrete Visual Tokenization} \label{sec:discrete visual tokenizer} 

Discrete visual tokenizers map images into sequences of discrete symbols that can subsequently be modeled by generative models~\citep{Oord2017NeuralDR,Esser2020TamingTF,Sun2024AutoregressiveMB, Tian2024VisualAM}. A typical discrete tokenizer consists of an encoder $\mathcal{E}_{\theta}$, a quantization module $\mathcal{Q}_{\phi}$, and a decoder $\mathcal{D}_{\varphi}$. Given an image $\bm{x}\in\mathbb{R}^{H\times W\times 3}$, the encoder produces continuous latent features 
\begin{align} 
\bm{z}_e = \mathcal{E}_{\theta}(\bm{x}) \in \mathbb{R}^{h\times w\times d}, \qquad h=H/f,\;\;w=W/f, \nonumber 
\end{align} 
where $f$ denotes the spatial downsampling factor and $d$ the latent dimension. 

At each spatial position $(i,j)$, the quantizer maps the continuous feature $\bm{z}^{ij}_e\in\mathbb{R}^{d}$ to a discrete representation,
\begin{align} r^{ij} = \mathcal{I}_{\phi}(\bm{z}^{ij}_e), \qquad \bm{z}^{ij}_q = \mathcal{Q}_{\phi}(\bm{z}^{ij}_e), \nonumber 
\end{align}
where $r^{ij}$ denotes the discrete symbol and $\bm{z}^{ij}_q$ the quantized representation. The resulting discrete symbols can be used as visual tokens for generative modeling, while the quantized latent is decoded as $\hat{\bm{x}} = \mathcal{D}_{\varphi}(\bm{z}_q)$. Different quantization algorithms impose different structures on the mapping from $\bm{z}_e$ to $\bm{z}_q$. In this work, we focus on three widely used families: vector quantization (VQ), product quantization (PQ), and scalar quantization (SQ), which we introduce below.

\subsection{Vector Quantization} 
\label{sec:vq} 
Vector quantization (VQ) discretizes an entire $d$-dimensional feature vector jointly~\citep{Oord2017NeuralDR}. Given a learnable codebook $\phi = \{\bm{e}_k\}_{k=1}^{K} \subset\mathbb{R}^{d},$ containing $K$ code vectors, VQ assigns each continuous feature $\bm{z}^{ij}_e$ to its nearest codebook entry: 
\vspace*{-3pt} 
\begin{align} r^{ij} = \argmin_{k\in\{1,\ldots,K\}} \|\bm{z}^{ij}_e-\bm{e}_k\|_2^2, \qquad \bm{z}^{ij}_q = \bm{e}_{r^{ij}}. \nonumber 
\end{align} 
Thus, each spatial feature is represented by one of the $K$ $d$-dimensional code vectors. 

A longstanding optimization challenge in VQ is \emph{codebook collapse}, where only a subset of the available code vectors is frequently selected~\citep{Dhariwal2020JukeboxAG,Takida2022SQVAEVB, Yu2021VectorquantizedIM,Lee2022AutoregressiveIG,Zheng2023OnlineCC}. The problem can become particularly pronounced for large codebooks~\citep{Zheng2023OnlineCC,Mentzer2023FiniteSQ}. 
A substantial body of work has therefore developed improved codebook updates, optimization procedures, and parameterizations to increase codebook usage~\citep{Razavi2019GeneratingDH,Williams2020HierarchicalQA,Zhang2023RegularizedVQ,Zhu2024ScalingTC,Zhu2024AddressingRC,Shi2024TamingScalable,Chang2025FullVQ,Lu2026BeyondStationarity}. Other work has studied difficulties caused by the non-differentiability of quantization, including the straight-through estimator and alternative gradient transformations~\citep{Huh2023StraighteningOT,Fifty2024RestructuringVQ}.

\subsection{Product Quantization} 
\label{sec:pq} 
Product quantization (PQ) factorizes a $d$-dimensional feature vector into $M$ lower-dimensional subvectors and quantizes each subvector independently ~\citep{Jgou2011ProductQF,Guo2024AddressingIC,Li2024ImageFolderAI}. Specifically,\vspace*{-5pt}
\begin{align} 
\bm{z}^{ij}_e = \bigoplus_{m=1}^{M} \bm{z}^{ij}_m, \qquad \bm{z}^{ij}_m\in\mathbb{R}^{d_m}, \qquad \sum_{m=1}^{M}d_m=d, \nonumber 
\end{align}\\[-5pt] 
where $\bigoplus$ denotes channel-wise concatenation. Each subvector is quantized using an independent subcodebook $\phi_m = \{\bm{e}_{m,k}\}_{k=1}^{n_m} \subset\mathbb{R}^{d_m}$ such that 
\begin{align} 
r^{ij}_m = \argmin_{k\in\{1,\ldots,n_m\}} \|\bm{z}^{ij}_m-\bm{e}_{m,k}\|_2^2, \qquad \hat{\bm{z}}^{ij}_m = \bm{e}_{m,r^{ij}_m}. \nonumber 
\end{align}\\[-5pt]
The full quantized feature is then $\bm{z}^{ij}_q = \bigoplus_{m=1}^{M} \hat{\bm{z}}^{ij}_m$. The $M$ subcodebooks jointly induce $K = \prod_{m=1}^{M}n_m$ possible composite codewords. Hence, PQ can represent a combinatorial discrete space using $\sum_{m=1}^{M}n_m$ explicitly stored subcodebook entries. For later analysis, we use $K$ to denote this \emph{composite code-space cardinality}, rather than the number of explicitly stored vectors.

\subsection{Scalar Quantization}
\label{sec:sq} 
Scalar quantization (SQ) independently discretizes scalar components of a latent representation and can be viewed as the maximally factorized case of PQ with $M=d$. Writing\vspace*{-5pt}
\begin{align} 
\bm{z}^{ij}_e = \bigoplus_{m=1}^{d}z^{ij}_m, \qquad z^{ij}_m\in\mathbb{R}, \nonumber 
\end{align}\\[-5pt]
each scalar component is quantized independently using a scalar codebook $\phi_m = \{e_{m,k}\}_{k=1}^{n_m} \subset \mathbb{R}$. The corresponding discrete index and quantized value are
\begin{align}
r^{ij}_m = \argmin_{k\in\{1,\ldots,n_m\}} |z^{ij}_m-e_{m,k}|^2, \qquad \hat{z}^{ij}_m = e_{m,r^{ij}_m}. \nonumber
\end{align}\\[-5pt]
The resulting quantized vector is $\bm{z}^{ij}_q = \bigoplus_{m=1}^{d}\hat{z}^{ij}_m$, with composite code-space cardinality \mbox{$K = \prod_{m=1}^{d} n_m.$}
Modern visual tokenizers instantiate this general scalar factorization in different ways. FSQ uses a small number of finite scalar levels per latent dimension~\citep{Mentzer2023FiniteSQ}. LFQ uses binary values for each quantized coordinate~\citep{yu2024language}, while BSQ combines binary quantization with hyperspherical normalization~\citep{zhao2024image}.

\section{Toward a Controlled Comparison of Quantization Algorithms}
\label{sec:controlled comparison}

Comparing quantization algorithms across independently trained visual tokenizers does not isolate the contribution of the quantizer itself. End-to-end tokenizer performance is jointly affected by the quantization module and the surrounding training system, including encoder--decoder capacity and architecture, discriminator design, training data, optimization schedule, and computational budget. These factors can vary substantially across existing tokenizer systems, while training each alternative quantizer from scratch under an otherwise identical setup is often computationally expensive. Consequently, empirical comparisons frequently rely on separately developed tokenizer systems whose experimental conditions are not fully matched. 

As illustrated in \Cref{table:unfair factors}, representative tokenizers differ substantially along these dimensions ~\citep{Esser2020TamingTF,Lee2022AutoregressiveIG,Zhu2024ScalingTC, Sun2024AutoregressiveMB,Tian2024VisualAM,Li2024ImageFolderAI, Ma2025UniTokAU}. As a result, differences in reconstruction performance may reflect not only the quantization algorithm, but also changes in model capacity, training objectives, data, or optimization resources. Such comparisons remain informative for evaluating complete tokenizer systems, but are insufficient for isolating the \emph{intrinsic effectiveness} of VQ, PQ, and SQ.

\begin{table*}[!t]
\centering
\vspace{-7pt}
\caption{\small{Representative implementation differences among discrete visual tokenizers. Differences in model design, training setup, and computational budget can confound direct comparisons of quantization algorithms. `-' denotes unreported information, and ``Para'' denotes the number of encoder--decoder parameters.}}
\vspace{-8pt}
\resizebox{\textwidth}{!}{
\begin{tabular}{l|ccccccc}
\toprule
Tokenizers & Para & Architectures & Discriminator & Training Datasets & Training Epochs & Training GPU Hours \\
\midrule
VQGAN~\citep{Esser2020TamingTF} & 68M & CNN U-Net & PatchGAN & ImageNet-1k & - & - \\
RQVAE~\citep{Lee2022AutoregressiveIG} & 95M & CNN U-Net & PatchGAN & ImageNet-1k & 50 & - \\
VQGAN-LC~\citep{Zhu2024ScalingTC} & 68M & CNN U-Net & PatchGAN & ImageNet-1k & 20 & 32 $\times$ V100 - Hours \\
Llama GEN~\citep{Sun2024AutoregressiveMB} & 68M & CNN U-Net & PatchGAN & ImageNet-1k & 40 & 2 $\times$ A100  200 Hours\\
VAR~\citep{Tian2024VisualAM} & 104M & CNN U-Net & StyleGAN & OpenImages & 16 & 16 $\times$ A100 60 Hours\\
ImageFolder~\citep{Li2024ImageFolderAI}  & - & Transformer SEED & StyleGAN & ImageNet-1k & 200 & 32 $\times$ A100 40 Hours \\
UniTok~\citep{Ma2025UniTokAU} & - & Transformer SEED & StyleGAN & OpenImages & - & 256 $\times$ A100 50 Hours \\
\bottomrule
\end{tabular}}
\vspace{-10pt}
\label{table:unfair factors}
\end{table*}

These confounders motivate three fundamental questions about intrinsic quantization effectiveness: \textbf{(Q1)} What should a quantization algorithm fundamentally optimize? \textbf{(Q2)} What conditions are required for a fair intrinsic comparison of different quantization algorithms? \textbf{(Q3)} Under these controlled conditions, how do VQ, PQ, and SQ compare in achievable distortion? In the next section, we develop a unified rate-distortion perspective to answer these three
questions. 

\newcommand{\quantizer}[1]{\mathcal{Q}_{#1}}
\newcommand{\optimalErr}{\mathcal{E}^*}
\newcommand{\optimalErrVQ}{\mathcal{E}^*_{\text{VQ}}}
\newcommand{\optimalErrPQ}{\mathcal{E}^*_{\text{PQ}}}
\newcommand{\optimalErrSQ}{\mathcal{E}^*_{\text{SQ}}}

\section{A Rate-distortion Perspective}
We now develop a unified rate--distortion framework to address the three questions posed in \Cref{sec:controlled comparison}. We first formalize the rate and distortion of discrete quantization, and then study its fundamental optimization objective, the conditions required for fair intrinsic comparison, and the achievable distortion of VQ, PQ, and SQ under controlled conditions.

\subsection{Fixed-Length Rate and Quantization Distortion}
\label{sec:rate_distortion_definition}

We formulate discrete quantization as a fixed-rate lossy compression problem. Unlike entropy-coded neural compression, where the rate is determined by the expected bitstream length under a learned entropy model~\citep{Balle2017EndToEnd,Jiang2026RDVQ}, our goal is to compare the intrinsic effectiveness of different quantization structures under a controlled representation budget. We therefore characterize the representational budget using the nominal fixed-length coding rate and measure the information lost through quantization using squared-error distortion.

\paragraph{Nominal fixed-length rate.}
Consider a tokenizer that produces $T$ discrete tokens, each taking values in a discrete space of cardinality $K$. Its nominal fixed-length coding rate is
\begin{align}
    R = T\log_2 K \quad \text{bits}.
    \label{eq:nominal_rate}
\end{align}
Equivalently, each token carries a nominal coding budget of $\log_2 K$ bits. 
This quantity measures the representational budget available to the quantizer under fixed-length coding and should be distinguished from the expected entropy-coded rate, which additionally depends on the probability distribution of the discrete symbols.

For VQ, $K$ is the number of available code vectors. For PQ and SQ, $K$ denotes the cardinality of the \emph{composite} discrete code space rather than the number of explicitly stored code vectors. Specifically, if the representation is factorized into $M$ independently quantized components with sub-codebook sizes $\{n_m\}_{m=1}^M$, then $K = \prod_{m=1}^{M} n_m$. Thus, VQ, PQ, and SQ can be compared under the same nominal coding \mbox{rate even though they impose different structural constraints on the discrete representation.}

\paragraph{Quantization distortion.}
Let $X\sim P$ denote a latent feature vector in $\mathbb{R}^d$, and let $\mathcal{Q}_{\phi}$ denote a quantizer that maps $X$ to a quantized representation $\mathcal{Q}_{\phi}(X)$. We measure quantization distortion using the expected squared Euclidean error
\begin{align}
    \mathcal{E}(P,\mathcal{Q}_{\phi})
    =
    \mathbb{E}_{X\sim P}
    \left[
        \left\|X-\mathcal{Q}_{\phi}(X)\right\|_2^2
    \right].
    \label{eq:quantization_distortion}
\end{align}
In empirical evaluations, this expectation is approximated by averaging the squared quantization error over the observed latent features.
This distortion directly measures the information lost by the quantization operation and serves as the central quantity in our subsequent analysis.

\begin{remark}[Rate allocation]
\label{remark:rate-allocation}
Under the nominal fixed-length rate $R=T\log_2 K$, doubling the token count is equivalent to squaring the code-space cardinality in terms of representation rate:
\begin{align}
    2T\log_2 K
    =
    T\log_2 K^2.
    \label{eq:rate_equivalence}
\end{align}
\end{remark}
This equivalence concerns the nominal coding budget only; equal rates do not in general imply identical distortion, since different allocations of the same rate across token count and code-space cardinality may induce different representation structures. In our controlled comparisons, we therefore match both the token count and the composite code-space cardinality whenever possible, providing a stronger control than matching $R$ alone.

\subsection{Distortion Minimization: The Primary Optimization Objective} 
\label{sec:information loss minimization} 
 
A substantial line of work on VQ has focused on mitigating codebook collapse and improving codebook utilization~\citep{Zhu2024ScalingTC,Dhariwal2020JukeboxAG,Williams2020HierarchicalQA,Zheng2023OnlineCC,Zhang2023RegularizedVQ,Ramesh2021ZeroShotTG}. In this work, we contend that minimizing distortion constitutes a more fundamental optimization objective. Intuitively, lower distortion promotes more stable and faithful representations in lossy compression systems~\citep{Touchette1999InformationtheoreticLO,Tomar2017InvarianceFE,Touchette2001InformationtheoreticAT}.  

We first show that distortion controls the local gradient discrepancy induced
by the straight-through estimator (STE). As derived in
Appendix~\ref{appendix:gradient gap distortion}, under standard local
smoothness, the squared gradient discrepancy $\mathcal{G}$ satisfies
\begin{align}
    \mathcal{G}
    =
    \mathcal{O}\!\left(\|z_e-z_q\|_2^2\right).
\end{align}
Since $\|z_e-z_q\|_2^2$ is exactly the squared-error distortion, this establishes a local connection between quantization distortion and STE gradient fidelity:
as the distortion approaches zero, the gradient discrepancy also vanishes. Thus, reducing distortion makes the gradient evaluated at the quantized representation locally closer to that evaluated at the continuous
representation, providing a mechanism through which lower distortion can contribute to more stable optimization.
 
We further formalize the relationship between distortion minimization and codebook utilization. Let $X$ be a random variable supported on $\mathcal{X}\subseteq\mathbb{R}^d$. A deterministic quantizer with codebook size $K\in\mathbb{N}_+$ is a mapping $f:\mathcal{X}\to\{1,\dots,K\}$. Define
\begin{align} 
\mathcal{E}(f) 
&\triangleq 
\min_{\hat{x}:\{1,\dots,K\}\to\mathbb{R}^d} 
\mathbb{E}\!\left[ 
\|X-\hat{x}(f(X))\|_2^2 
\right],  \quad  
\mathrm{U}(f) \triangleq 
\frac{ 
\bigl| 
\{k:\mathbb{P}(f(X)=k)>0\} 
\bigr| 
}{K}. 
\nonumber 
\end{align} 

\begin{assumption}[Non-degenerate splittability]
\label{assumption:splittability}
For any measurable set $A \subseteq \mathcal{X}$ with
$\mathbb{P}(X\in A)>0$ and nonzero conditional variance,
there exists a partition of $A$ into two measurable subsets
$A=A_1\dot{\cup}A_2$, each with positive probability, such that
\[
\mathbb{E}[X\mid X\in A_1]
\neq
\mathbb{E}[X\mid X\in A_2].
\]
\end{assumption}
This mild condition is satisfied by standard continuous distributions with densities. The following proposition further excludes the trivial setting in which the support of the distribution can already be represented exactly using at most $K$ codewords, since in that case zero distortion may be achieved without utilizing the entire codebook. For example, for an empirical distribution supported on $n$ distinct samples, $K<n$ is sufficient to rule out such a zero-distortion solution.
 
\begin{proposition} 
\label{theorem:distortion-vs-codebook-utilization} 
(a) Let\vspace*{-9pt}
\begin{align}
f^\star 
\in 
\argmin_{f:\mathcal{X}\to\{1,\dots,K\}} 
\mathcal{E}(f).    
\end{align}\\[-8pt]
Under Assumption~\ref{assumption:splittability}, if no zero-distortion quantizer with at most $K$ codewords exists, then every global distortion minimizer satisfies
$ 
\mathrm{U}(f^\star)=1. 
$ 
 
(b) There exist quantizers 
$g:\mathcal{X}\to\{1,\dots,K\}$ 
with $ 
\mathrm{U}(g)=1 $ 
that are not distortion minimizers:\vspace*{-3pt}
\begin{align}
\mathcal{E}(g) 
> 
\min_{f:\mathcal{X}\to\{1,\dots,K\}} 
\mathcal{E}(f). \nonumber
\end{align}\\[-28pt]
\end{proposition} 
Proposition~\ref{theorem:distortion-vs-codebook-utilization} establishes that, under these mild conditions, every global distortion minimizer fully utilizes the codebook, whereas full utilization does not imply distortion optimality. The proof is provided in Appendix~\ref{appendix:proof-distortion-codebook-utilization}.

\subsection{Conditions for Fair Intrinsic Comparison}
\label{sec:fair comparison}

A fair intrinsic comparison of quantization algorithms requires controlling both the latent source distribution and coding rate. First, all quantizers must operate on the same latent distribution, because quantization distortion depends not only on the quantization algorithm but also on the statistical properties of the source distribution. Even under the simplest change in source statistics, namely a global rescaling $X' = aX$, the optimal $K$-point squared-error distortion changes as\vspace*{-4pt}
\begin{align}
    \mathcal{E}^\star(aX,K)
    =
    a^2\mathcal{E}^\star(X,K).
    \label{eq:distortion_scaling}
\end{align}\\[-15pt]
Thus, differences in latent feature statistics can directly confound distortion comparisons. This result is formally established in Proposition~\ref{prop:distortion-scaling} in Appendix~\ref{appendix:error-variance-relationship}.

\begin{condition}[Matched latent distribution]
\label{condition:feature distribution}
All compared quantizers must be evaluated on the same latent distribution $P$.
\end{condition}

Second, distortion must be compared under the same coding budget. For $T$ discrete tokens with code-space cardinality $K$, the nominal fixed-length coding rate is $R = T\log_2 K$. Since both $T$ and $K$ affect representational capacity and reconstruction behavior~\citep{yu2024language,Zhu2024ScalingTC,Yu2024TiTok,Bachmann2025FlexTok}, comparisons at different rates do not isolate intrinsic quantization effectiveness.

\begin{condition}[Matched nominal rate]
\label{condition:rate constraints}
The nominal fixed-length coding rate must be identical across all compared algorithms.
\end{condition}

In our controlled experiments, we adopt the stronger criterion
$T_A=T_B$ and $K_A=K_B$, rather than matching only
$T\log_2K$, to avoid confounding different allocations of the same nominal rate. 

\subsection{Rate--Distortion Analysis of VQ, PQ, and SQ}
\label{sec:quantization error analysis}

In this section, we study the rate--distortion behavior of idealized VQ, PQ, and SQ, and compare their distortion under equal coding rates.

Given a probability distribution $\Prob$ over $\real^\usedim$ and a codebook size $K \in \mathbb{N}_+$, we define the optimal distortions for VQ, PQ, and SQ as follows:
\begin{align*} 
    \optimalErrVQ (\Prob, K) & \mydefn \inf_{\phi} \Big\{ \Exs \big[ \vecnorm{X - \quantizer{\phi}(X)}{2}^2 \big] : \abss{\phi} \leq K \Big\},\\
    \optimalErrPQ \big(\Prob, K, M \big) & \mydefn \inf_{\phi} \Big\{ \Exs \big[ \vecnorm{X - \quantizer{\phi}(X)}{2}^2 \big] : \phi = \bigoplus_{m=1}^M \phi_m,
    \phi_m \subseteq \real^{d_m}, \abss{\phi_m} = n_m, \prod_{m = 1}^{M} n_m \leq K \Big\},
    \\
    \optimalErrSQ \big(\Prob, K \big) & \mydefn \inf_{\phi} \Big\{ \Exs \big[ \vecnorm{X - \quantizer{\phi}(X)}{2}^2 \big] : \phi = \bigoplus_{m=1}^\usedim \phi_m,
    \phi_m \subseteq \real, \abss{\phi_m} = n_m, \prod_{m = 1}^{\usedim} n_m \leq K  \Big\}.
\end{align*}\\[-8pt]
Here $\abss{\phi}$ denotes the size of the codebook $\phi$, and $\bigoplus$ denotes the Cartesian product of sets.

Clearly, SQ is a special case of PQ with $M = d$, and PQ is a special case of VQ, a classical structural relationship in quantization theory. Thus, the optimal distortions satisfy:
\begin{align*}
    \optimalErrVQ (\Prob, K) \leq \optimalErrPQ \big(\Prob, K,  M \big) \leq \optimalErrSQ \big(\Prob, K \big), \; \mbox{for any } 2\leq M \leq d.
\end{align*}
Let $\mathcal{P}$ be the set of all probability distributions over $[- 1, 1]^\usedim$. The following results provide quantitative characterizations of the optimal distortions for VQ, PQ, and SQ.

\begin{proposition}\label{thm:worst-case}
    For any $K \geq 2^\usedim$, we have
    \begin{align*} 
      \frac{d}{8K^{2/d}} \leq \sup_{\Prob \in \mathcal{P}} \optimalErrVQ (\Prob, K) \leq \sup_{\Prob \in \mathcal{P}} \optimalErrSQ \big(\Prob, K \big) \leq \frac{4d}{K^{2/d}}
    \end{align*}
\end{proposition}
\noindent See \Cref{appendix:proof-thm-worst-case} for the proof. 
\Cref{thm:worst-case} shows that for worst-case distributions, the optimal distortions of VQ, PQ, and SQ are the same up to universal constant factors. All three methods achieve a distortion that scales as $\Theta(d/K^{2/d})$.

On the other hand, when the data distribution has intrinsic low-dimensional structures, VQ can achieve strictly better distortion scaling than PQ and SQ. We illustrate this phenomenon in the following proposition.
\begin{proposition}\label{thm:adaptive}
    If the support of $\Prob \in \mathcal{P}$ is contained in a $d_\text{eff}$-dimensional subspace of $\real^d$ with $d_\text{eff} < d$, then for any $K \geq 2^{d_\text{eff}}$, we have\vspace*{-8pt}
    \begin{align*}
        \optimalErrVQ (\Prob, K) \leq \frac{4 d d_\text{eff}}{K^{2/d_\text{eff}}}.
    \end{align*}
    On the other hand, there exists a 1-dimensional linear subspace $L \subseteq \real^d$ and a distribution $\Prob$ supported on $L \cap [-1, 1]^d$ such that for any $K \geq 2^d$, we have
    \begin{align*}
        &\optimalErrPQ \big(\Prob, K, M \big) \geq \frac{M}{4} K^{- 2/M}, \quad \mbox{for any } 2 \leq M \leq d,
        \quad
        \mbox{and}
        \quad
        \optimalErrSQ \big(\Prob, K \big) \geq \frac{d}{4} K^{- 2/d}.
    \end{align*}
\end{proposition}
\noindent See \Cref{appendix:proof-thm-adaptive} for the proof. \Cref{thm:adaptive} shows that when the data distribution has intrinsic dimension $d_\text{eff}<d$, VQ can achieve distortion scaling as $\mathcal{O}(K^{-2/d_\text{eff}})$, significantly smaller than the worst-case rate $\Theta(d/K^{2/d})$ when $d_\text{eff}\ll d$. In contrast, for a simple 1-dimensional distribution, SQ still suffers $\Omega(d/K^{2/d})$ while VQ can achieve $\mathcal{O}(dK^{-2})$. PQ interpolates between these extremes, with performance depending on the number of blocks $M$. This separation goes beyond the nested family ordering: VQ can exploit intrinsic low-dimensional structure, whereas fixed coordinate-wise factorization in PQ and SQ may prevent such exploitation. The VQ upper-bound argument also extends to supports with covering-number growth $N(\varepsilon)\lesssim C\varepsilon^{-d_\text{eff}}$, with constants depending on covering regularity; the linear-subspace case above is the cleanest special case.  Appendix~\ref{appendix:classical-quantization-theory} discusses the relationship of these results to classical quantization theory and clarifies the novelty boundary.

\begin{table*}[!h]
\centering
\caption{\small{Quantization and reconstruction performance of VQ, PQ, and SQ on ImageNet-1K under the controlled latent-space setting. $^{\dagger}$: Results cited from VQ-Transplant~\citep{Fang2026VQTransplant}.
Within each quantization type and phase (Substitution/Adaptation), optimal values are \underline{underlined}; overall best results per metric are \textbf{bold}.}}
\vspace*{-5pt}
\label{tab:vq pq sq}
\resizebox{0.93\textwidth}{!}{%
\begin{tabular}{lccccccccccc}
\toprule
Approaches  & Types & Phase & Tokens & $K$ & $\mathcal{E}$($\downarrow$) & $\mathrm{U}$ ($\uparrow$)  & PSNR($\uparrow$) & SSIM($\uparrow$) & LPIPS ($\downarrow$) & r-FID($\downarrow$) & r-IS($\uparrow$) \\
\midrule
Vanilla VQ$^{\dagger}$ & VQ & Substitution & 512 & 65536 & 0.422 & 0.2\% & 22.04 & 53.1 & 0.243 & 10.89 & 103.8\\
EMA VQ$^{\dagger}$ & VQ & Substitution & 512 & 65536 & 0.217 & 65.5\% & 24.94 & 65.9 & 0.127 & 1.78 & 185.8\\
Online VQ$^{\dagger}$ & VQ & Substitution & 512 & 65536 & 0.280 & 13.5\% & 24.42 & 63.2 & 0.147 & 2.28 & 174.2\\
Wasserstein VQ$^{\dagger}$ & VQ & Substitution & 512 & 65536 & \textbf{\underline{0.201}} & 99.6\% & 25.22 & \textbf{\underline{66.9}} & \textbf{\underline{0.121}} & 1.76 & 186.0\\
MMD VQ$^{\dagger}$ & VQ & Substitution & 512 & 65536  & \textbf{\underline{0.201}} & \underline{99.9\%} & \textbf{\underline{25.24}} & 66.8 & \textbf{\underline{0.121}} & \underline{1.69} & \underline{187.3} \\
\midrule
Vanilla VP2 & PQ & Substitution & 512 & 65536 & 0.233 & 59.2\% & 24.80 & 65.3 & 0.130 & 1.84 & 183.1 \\
EMA VP2 & PQ & Substitution & 512 & 65536 & \underline{0.209} & \textbf{\underline{100\%}} & 25.08 & \underline{66.4} & \underline{0.123} & 1.68 &  187.2\\
Online VP2 & PQ & Substitution & 512 & 65536 & 0.211 & \textbf{\underline{100\%}} & \underline{25.09} & 66.3 & 0.124 & 1.79 & 185.7 \\ 
Wasserstein VP2 & PQ & Substitution & 512 & 65536 & 0.217 & \textbf{\underline{100\%}} & 24.79 & 65.8 & 0.128 & 1.78 & 185.7 \\ 
MMD VP2 & PQ & Substitution & 512 & 65536 & 0.212 & \textbf{\underline{100\%}} & 25.00 & 66.1 & \underline{0.123} & \textbf{\underline{1.61}} &  \textbf{\underline{189.5}}\\
\midrule
FSQ & SQ & Substitution & 512 & 65536 & 0.272 & \underline{100\%} & 24.46 & 62.5 & 0.140 & 2.35 & 174.8 \\
LFQ & SQ & Substitution & 512 & 65536 & 0.279 & 29.8\% & 24.06 & 62.6 & 0.146 & 2.15 & 176.2 \\
BSQ & SQ & Substitution & 512 & 65536 & \underline{0.231} & \textbf{\underline{100\%}}  & \underline{24.55} & \underline{65.2} & \underline{0.132} & \underline{1.96} & \underline{182.1} \\
\midrule
\midrule
Vanilla VQ$^{\dagger}$ & VQ & Adaptation & 512 & 65536 & 0.422 & 0.2\% & 21.19 & 50.7 & 0.209  & 5.05 & 118.9\\
EMA VQ$^{\dagger}$ & VQ & Adaptation & 512 & 65536 & 0.217 & 65.5\% & 24.36 & 64.1 & 0.111 & 0.99 & 194.3\\
Online VQ$^{\dagger}$ & VQ & Adaptation & 512 & 65536 & 0.280 & 13.5\% & 23.84 & 61.6 & 0.130 & 1.38 & 182.9 \\
Wasserstein VQ$^{\dagger}$ & VQ & Adaptation & 512 & 65536 & \textbf{\underline{0.201}} & 99.6\% & \textbf{\underline{24.68}} & \textbf{\underline{65.4}} & \textbf{\underline{0.106}} & 0.92 & 195.5 \\
MMD VQ$^{\dagger}$ & VQ & Adaptation & 512 & 65536  & \textbf{\underline{0.201}} & \underline{99.9\%} & 24.65 & 65.0 & \textbf{\underline{0.106}} & \textbf{\underline{0.86}} & \textbf{\underline{197.1}}\\
\midrule
Vanilla VP2 & PQ & Adaptation & 512 & 65536 & 0.233 & 59.2\% & 24.28 & 64.0 & 0.114 & 1.07 & 191.7 \\
EMA VP2 & PQ & Adaptation & 512 & 65536 & \underline{0.209} & \textbf{\underline{100\%}} & \underline{24.55} & \underline{64.9} & \underline{0.107} & \underline{0.93} & 195.4 \\
Online VP2 & PQ & Adaptation & 512 & 65536 & 0.211 & \textbf{\underline{100\%}} & 24.53 & 64.7 & 0.108 & 0.95 & 195.3 \\ 
Wasserstein VP2 & PQ & Adaptation & 512 & 65536 & 0.217 & \textbf{\underline{100\%}} & 24.44 & 64.6 & 0.110 & 0.99 & 193.5 \\ 
MMD VP2 & PQ & Adaptation & 512 & 65536 & 0.212 & \textbf{\underline{100\%}} & 24.43 & 64.5 & 0.109 & 0.95 & \underline{196.1} \\
\midrule
FSQ & SQ & Adaptation & 512 & 65536 & 0.272 & \textbf{\underline{100\%}} & 23.87 & 61.0 & 0.119 & 1.23 & 186.6 \\
LFQ & SQ & Adaptation & 512 & 65536 & 0.279 & 29.8\% & 23.42 & 60.7 & 0.130 & 1.30 & 183.2 \\
BSQ & SQ & Adaptation & 512 & 65536 & \underline{0.231} & \textbf{\underline{100\%}} & \underline{24.06} & \underline{63.6} & \underline{0.117} & \underline{1.07} & \underline{190.8} \\
\bottomrule
\end{tabular}
}
\vspace{-5pt}
\end{table*}

\section{Empirical Evaluation}
\label{sec:experiments}
\vspace*{-4pt}
We empirically compare VQ, PQ, and SQ under matched latent distributions and nominal coding rates, with primary experiments conducted on ImageNet-1K~\citep{Deng2009ImageNetAL}. We further examine rate-distortion behavior across coding budgets and validate the conclusions across datasets and in a controlled pixel-space setting.

\vspace*{-3pt}
\subsection{Experimental Setup}
\label{sec:experimental-setup}
\vspace*{-3pt}

\paragraph{Controlled latent-space comparison.}
For our primary experiments, we use a pretrained VAR tokenizer~\citep{Tian2024VisualAM} within the VQ-Transplant framework~\citep{Fang2026VQTransplant}. During quantization-module substitution, the pretrained encoder and decoder are frozen while the original quantizer is replaced by the method under evaluation. During decoder adaptation, the encoder and transplanted quantizer are frozen while the decoder adapts to the quantized representation. Since all methods share the same pretrained encoder and representation pipeline, they operate on the same latent source distribution, satisfying the matched-distribution condition in \Cref{sec:fair comparison}.

\textbf{Quantization algorithms.}~~~
We compare VQ, PQ, and SQ. For VQ, we evaluate Vanilla VQ~\citep{Oord2017NeuralDR}, EMA VQ~\citep{Razavi2019GeneratingDH}, Online VQ~\citep{Zheng2023OnlineCC}, Wasserstein VQ~\citep{Fang2026distributional}, and MMD VQ~\citep{Fang2026VQTransplant}. For PQ, we use the corresponding Vanilla VP2, EMA VP2, Online VP2, Wasserstein VP2, and MMD VP2 variants. For SQ, we evaluate FSQ~\citep{Mentzer2023FiniteSQ}, LFQ~\citep{yu2024language}, and BSQ~\citep{zhao2024image}. In the primary latent-space comparison, all methods use $T=512$ tokens and composite code-space cardinality $K=65{,}536$, matching both components of the nominal rate $R=T\log_2K$. Full implementation and training details for the latent- and pixel-space experiments are provided in Appendix~\ref{appendix:experimental details}.

\textbf{Evaluation metrics.}~~~
We report quantization distortion $\mathcal{E}$ and codebook utilization $\mathrm{U}$. Reconstruction quality is evaluated using Peak Signal-to-Noise Ratio (PSNR), Structural Similarity Index (SSIM), LPIPS~\citep{Zhang2018TheUE}, reconstruction Fr\'echet Inception Distance (r-FID)~\citep{Heusel2017GANsTB}, and reconstruction Inception Score (r-IS)~\citep{Salimans2016ImprovedTF}.

\vspace*{-3pt}
\subsection{Experimental Results}
\label{sec:experimental-results}
\vspace*{-3pt}

\textbf{Controlled quantization comparison in latent space.}~
Table~\ref{tab:vq pq sq} reports the primary ImageNet-1K comparison under matched latent feature statistics, token count, and code-space cardinality. A clear distortion hierarchy emerges: the best VQ method achieves lower distortion than the best PQ method, which in turn outperforms the best SQ method. Specifically, MMD VQ attains $\mathcal{E}=0.201$, compared with $0.209$ for PQ and $0.231$ for SQ, consistent with the theoretical ordering
$\mathcal{E}^{\star}_{\text{VQ}}
\leq
\mathcal{E}^{\star}_{\text{PQ}}
\leq
\mathcal{E}^{\star}_{\text{SQ}}$
in \Cref{sec:quantization error analysis}. Lower distortion also yields better reconstruction after decoder adaptation: MMD VQ achieves the best r-FID of $0.86$, compared with $0.93$ for PQ and $1.07$ for SQ. Meanwhile, several PQ and SQ methods attain $100\%$ codebook utilization despite higher distortion, showing that high utilization alone does not imply stronger quantization effectiveness.

\begin{wrapfigure}[11]{r}{0.39\textwidth}
    \vspace{-1pt}
    \begin{center}
    \includegraphics[width=0.37\textwidth]{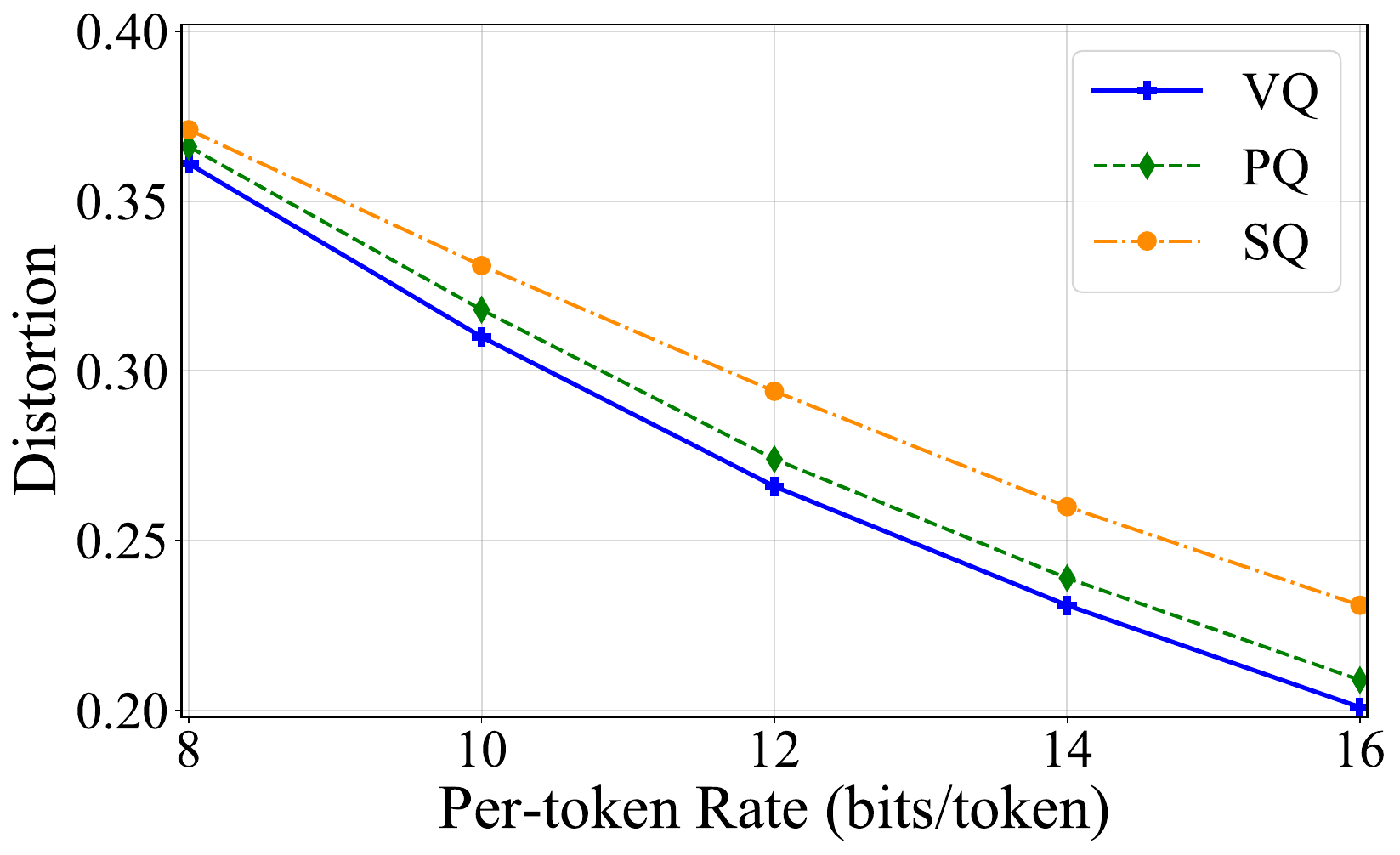}
    \vspace{-8pt}
    \caption{\small{
    Rate–distortion curves on ImageNet-1K, reporting the best result within each quantization family at each rate.
    }}
    \label{fig:rd-curve}
    \end{center}
    \vspace{-8pt}
\end{wrapfigure}
\paragraph{Rate-distortion curves.}
Figure~\ref{fig:rd-curve} compares the best-performing method within each quantization family across different coding rates. Consistent with our family-level theoretical comparison, at each nominal per-token rate $\log_2 K$, we report the minimum distortion achieved among the evaluated variants within each quantization family. Increasing the coding rate monotonically reduces distortion for all three quantization families. More importantly, the ordering remains consistent across the evaluated range: VQ achieves the lowest distortion at every operating point, followed by PQ and SQ. This persistent separation across representation budgets is consistent with the theoretical hierarchy in \Cref{sec:quantization error analysis} and shows that VQ's distortion advantage is not specific to a particular code-space cardinality.

\paragraph{Distortion versus codebook utilization.}
We further examine whether distortion or codebook utilization is more strongly associated with reconstruction fidelity. Using Table~\ref{tab:vq pq sq}, we compute Spearman's rank correlations with r-FID. Distortion shows a near-perfect positive correlation with r-FID ($\rho=0.996$, $p<10^{-8}$), whereas codebook utilization exhibits a substantially weaker negative correlation ($\rho=-0.540$, $p=0.057$). Thus, distortion is considerably more informative of reconstruction fidelity than utilization, supporting the distortion-minimization principle in \Cref{sec:information loss minimization}: high utilization alone does not ensure faithful preservation of the source representation.

\paragraph{Additional rate-distortion analyses.}
We further analyze the effects of code-space cardinality $K$, token count $T$, and matched-rate allocations. Increasing either $K$ or $T$ reduces distortion, while equal-rate configurations yield similar but not identical performance. See Appendix~\ref{appendix:codebook size and tokens} for full results.

\paragraph{Generalization across datasets and representation spaces.} We further verify that our findings generalize beyond ImageNet-1K and the pretrained VAR latent space. Under the same controlled latent-space protocol, VQ retains its advantage over PQ and SQ on FFHQ and CelebA-HQ, with distortion remaining more strongly associated with reconstruction fidelity than codebook utilization (Appendix~\ref{appendix:generalization datasets}). A controlled pixel-space comparison on CelebA-HQ yields the same qualitative conclusions, with VQ achieving the lowest distortion and best reconstruction fidelity (Appendix~\ref{appendix:pixel space}). These results demonstrate robustness across datasets and representation spaces.
\vspace*{-3pt}
\section{Conclusion}
\label{sec:conclusion}
\vspace*{-3pt}

In this work, we presented a unified rate-distortion perspective for understanding and comparing vector, product, and scalar quantization. We characterize quantization error as distortion and the nominal fixed-length coding budget through token count and discrete code-space cardinality. This perspective yields three main conclusions. First, distortion minimization is a more fundamental optimization principle than codebook utilization: under mild conditions, it implies full utilization and controls the local STE-induced gradient discrepancy. Second, fair intrinsic comparison requires matched latent distributions and nominal coding rates. Third, beyond the classical ordering
$\mathcal{E}^{\star}_{\text{VQ}}
\leq\mathcal{E}^{\star}_{\text{PQ}}
\leq\mathcal{E}^{\star}_{\text{SQ}}$, VQ can exploit low-dimensional source structure to achieve substantially
better distortion scaling than fixed coordinate-factorized PQ and SQ. Our controlled experiments support these conclusions: distortion is more strongly associated with reconstruction fidelity than codebook utilization, and modern VQ methods achieve lower distortion and better reconstruction performance than PQ and SQ under matched representation budgets. Overall, our results provide a principled framework for evaluating intrinsic quantization effectiveness and for guiding the analysis and design of future discrete representation learning systems.

\section*{Acknowledgements}

Dehan Kong was partially supported by the Natural Sciences and Engineering Research Council of Canada (NSERC) Discovery Grant RGPIN-2022-04646.

\printbibliography

\clearpage

\begin{appendices}
\section{Proof of \Cref{theorem:distortion-vs-codebook-utilization}}
\label{appendix:proof-distortion-codebook-utilization}

We prove the relationship between distortion minimization and codebook utilization. Let $X$ be a random variable supported on $\mathcal{X}\subseteq\mathbb{R}^d$.
A deterministic quantizer with codebook size $K$ is a measurable mapping $ f:\mathcal{X}\to\{1,\dots,K\}.$
Given a reconstruction mapping $\hat{x}:\{1,\dots,K\}\to\mathbb{R}^d,$ define the squared-error distortion as
\[
\mathcal{E}(f,\hat{x})
\coloneqq
\mathbb{E}\!\left[
\|X-\hat{x}(f(X))\|_2^2
\right].
\]

For a fixed quantizer $f$, the optimal reconstruction is given by the conditional mean $ \hat{x}_f(k) = \mathbb{E}[X\mid f(X)=k],$
for every codeword $k$ with $\mathbb{P}(f(X)=k)>0$.
Substituting the optimal reconstruction yields
\[
\mathcal{E}(f)
\coloneqq
\min_{\hat{x}}
\mathcal{E}(f,\hat{x})
=
\mathbb{E}\!\left[
\|X-\mathbb{E}[X\mid f(X)]\|_2^2
\right].
\]

Equivalently,
$
\mathcal{E}(f)
=
\mathbb{E}\!\left[
\operatorname{Var}(X\mid f(X))
\right],
$
where for a vector-valued random variable,
\[
\operatorname{Var}(X\mid f(X))
\coloneqq
\mathbb{E}\!\left[
\|X-\mathbb{E}[X\mid f(X)]\|_2^2
\,\middle|\,
f(X)
\right].
\]

\begin{proof}
\textbf{Part (a).}
We prove by contradiction.

Suppose that $f^\star$ minimizes distortion but does not use all codewords.
Then there exists an unused codeword $k_{\mathrm{new}}$ such that
\[
\mathbb{P}(f^\star(X)=k_{\mathrm{new}})=0.
\]

Since no zero-distortion quantizer exists, we must have
$
\mathcal{E}(f^\star)>0.
$
Therefore, there exists at least one quantization cell
\[
A^\star
=
\{x\in\mathcal{X}:f^\star(x)=k^\star\}
\]
with positive probability and nonzero conditional variance.

By Assumption~\ref{assumption:splittability}, there exists a measurable
partition
\[
A^\star=A_1\dot{\cup}A_2
\]
such that both $A_1$ and $A_2$ have positive probability and
\[
\mu_i = \mathbb{E}[X\mid X\in A_i], \qquad \mu_1\neq\mu_2.
\]

Construct a refined quantizer $f'$ by splitting $A^\star$:
\[
f'(x)
=
\begin{cases}
k^\star,
& x\in A_1,
\\
k_{\mathrm{new}},
& x\in A_2,
\\
f^\star(x),
& x\notin A^\star.
\end{cases}
\]

For all cells other than $A^\star$, the distortion contribution remains unchanged.
Within $A^\star$, the law of total variance gives
\[
\operatorname{Var}(X\mid X\in A^\star)
=
\sum_{i=1}^2
\mathbb{P}(A_i\mid A^\star)
\operatorname{Var}(X\mid X\in A_i)
+
\sum_{i=1}^2
\mathbb{P}(A_i\mid A^\star)
\|\mu_i-\mu\|_2^2,
\]
where
$
\mu
=
\mathbb{E}[X\mid X\in A^\star].
$

Since $\mu_1\neq\mu_2$, the second term is strictly positive.
Hence
\[
\sum_{i=1}^2
\mathbb{P}(A_i\mid A^\star)
\operatorname{Var}(X\mid X\in A_i)
<
\operatorname{Var}(X\mid X\in A^\star).
\]
Therefore,
$
\mathcal{E}(f')
<
\mathcal{E}(f^\star),
$
contradicting the optimality of $f^\star$. Hence every distortion-minimizing quantizer must use all $K$ codewords:
$
\mathrm{U}(f^\star)=1.
$

\textbf{Part (b).}
We construct a full-utilization quantizer that is not distortion-optimal. Let $K=2$ and let $X\sim \mathrm{Unif}[0,1]$.
Consider the quantizer
\[
f_0(x)
=
\begin{cases}
1,
& x\in[0,1/2],
\\
2,
& x\in(1/2,1].
\end{cases}
\]
Its optimal reconstructions are
\[
\hat{x}_{f_0}(1)=\frac14,
\qquad
\hat{x}_{f_0}(2)=\frac34.
\]
Hence
\[
\mathcal{E}(f_0)
=
\int_0^{1/2}
\left(x-\frac14\right)^2 dx
+
\int_{1/2}^{1}
\left(x-\frac34\right)^2 dx
=
\frac{1}{48}.
\]

Now define another quantizer
\[
g(x)
=
\begin{cases}
1,
& x\in[0,9/10],
\\
2,
& x\in(9/10,1].
\end{cases}
\]
Both codewords are used, so
$
\mathrm{U}(g)=1.
$

The optimal reconstructions are
\[
\hat{x}_{g}(1)=\frac{9}{20},
\qquad
\hat{x}_{g}(2)=\frac{19}{20}.
\]
Therefore,
\[
\mathcal{E}(g)
=
\int_0^{9/10}
\left(x-\frac{9}{20}\right)^2 dx
+
\int_{9/10}^{1}
\left(x-\frac{19}{20}\right)^2 dx
=
\frac{73}{1200}.
\]

Since
\[
\frac{73}{1200}
>
\frac{1}{48},
\]
we obtain
\[
\min_{f:\mathcal{X}\to\{1,2\}}
\mathcal{E}(f)
\le
\mathcal{E}(f_0)
<
\mathcal{E}(g).
\]
Thus $g$ uses all codewords but is not distortion-optimal.
\end{proof}

\section{Proof of \Cref{thm:worst-case}}
\label{appendix:proof-thm-worst-case}
We first prove the upper bound by constructing an SQ scheme. Let
$n = \floor{K^{1/d}}$. Since $K \geq 2^d$, we have $K^{1/d} \geq 2$ and hence
$n \geq K^{1/d}/2$. For each coordinate $m\in[d]$, partition $[-1,1]$ into
$n$ intervals of equal length $2/n$ and choose their midpoints as scalar
reconstruction levels,
\begin{align*}
    \phi_m = \left\{-1 + \frac{2i-1}{n}: i=1,\ldots,n\right\}.
\end{align*}
Let $\phi=\bigoplus_{m=1}^d \phi_m$. Then $|\phi|=n^d\leq K$. For any
$x\in[-1,1]^d$, nearest-neighbor scalar quantization satisfies
$|x_m-Q_{\phi_m}(x_m)|\leq 1/n$ for every coordinate. Therefore,
\begin{align*}
    \Exs\!\left[\vecnorm{X-\quantizer{\phi}(X)}{2}^2\right]
    &\leq \sup_{x\in[-1,1]^d}\vecnorm{x-\quantizer{\phi}(x)}{2}^2 \\
    &\leq \frac{d}{n^2}
    \leq \frac{4d}{K^{2/d}}.
\end{align*}
Since this construction is valid for every $\Prob\in\mathcal P$,
\begin{align*}
    \sup_{\Prob\in\mathcal P}\optimalErrSQ(\Prob,K)
    \leq \frac{4d}{K^{2/d}}.
\end{align*}

We next prove the lower bound using the uniform distribution
$\Prob=\operatorname{Unif}([-1,1]^d)$. Let
$\phi=\{e_1,\ldots,e_K\}$ be an arbitrary VQ codebook, and let
$A_1,\ldots,A_K$ be its nearest-neighbor cells intersected with $[-1,1]^d$.
Write $v_k=|A_k|$ and let $V_d$ denote the volume of the Euclidean unit ball
in $\mathbb R^d$. For any measurable set $A\subseteq\mathbb R^d$ of volume
$v$ and any $e\in\mathbb R^d$, the centered Euclidean ball of volume $v$
minimizes the second moment about $e$, so
\begin{align*}
    \int_A \vecnorm{x-e}{2}^2\,dx
    \geq \frac{d}{d+2}V_d^{-2/d}v^{1+2/d}.
\end{align*}
Consequently,
\begin{align*}
    \Exs\!\left[\vecnorm{X-\quantizer{\phi}(X)}{2}^2\right]
    &=\frac{1}{2^d}\sum_{k=1}^K
      \int_{A_k}\vecnorm{x-e_k}{2}^2\,dx \\
    &\geq \frac{1}{2^d}\frac{d}{d+2}V_d^{-2/d}
      \sum_{k=1}^K v_k^{1+2/d}.
\end{align*}
Because $\sum_{k=1}^K v_k=2^d$, convexity gives
\begin{align*}
    \sum_{k=1}^K v_k^{1+2/d}
    \geq K\left(\frac{2^d}{K}\right)^{1+2/d}
    =2^{d+2}K^{-2/d}.
\end{align*}
Thus
\begin{align*}
    \Exs\!\left[\vecnorm{X-\quantizer{\phi}(X)}{2}^2\right]
    \geq \frac{4d}{d+2}V_d^{-2/d}K^{-2/d}.
\end{align*}
Finally, $V_d^{2/d}\leq 32/(d+2)$ for every integer $d\geq1$.
For $d=1,2$ this follows directly from $V_1=2$ and $V_2=\pi$; for
$d\geq3$, it follows from the standard bound
$V_d\leq(2\pi e/d)^{d/2}$ and
$2\pi e/d\leq32/(d+2)$. Hence
\begin{align*}
    \Exs\!\left[\vecnorm{X-\quantizer{\phi}(X)}{2}^2\right]
    \geq \frac{d}{8}K^{-2/d}.
\end{align*}
Since the codebook $\phi$ was arbitrary,
\begin{align*}
    \sup_{\Prob\in\mathcal P}\optimalErrVQ(\Prob,K)
    \geq \frac{d}{8K^{2/d}}.
\end{align*}
Combining this with
$\optimalErrVQ(\Prob,K)\leq\optimalErrSQ(\Prob,K)$ completes the proof.

\section{Proof of \Cref{thm:adaptive}}
\label{appendix:proof-thm-adaptive}
We first prove the VQ upper bound. By assumption, there exists a
$d_\text{eff}$-dimensional linear subspace $L\subseteq\real^d$ such that
$\Prob$ is supported on $L\cap[-1,1]^d$. Let
$U\in\real^{d\times d_\text{eff}}$ have orthonormal columns spanning $L$.
For $X\sim\Prob$, write $X=UZ$ with $Z=U^\top X$, and let
$\widetilde{\Prob}$ denote the distribution of $Z$. Since $U$ is an
isometry on $\real^{d_\text{eff}}$,
\begin{align*}
    \vecnorm{Z}{2}=\vecnorm{UZ}{2}=\vecnorm{X}{2}\leq\sqrt d,
\end{align*}
and hence $\vecnorm{Z}{\infty}\leq\sqrt d$. Therefore,
$\widetilde{\Prob}$ is supported on
$[-\sqrt d,\sqrt d]^{d_\text{eff}}$.

Let $n=\floor{K^{1/d_\text{eff}}}$. Because
$K\geq2^{d_\text{eff}}$, we have
$n\geq K^{1/d_\text{eff}}/2$. Quantize each coordinate of $Z$
independently using $n$ equally spaced intervals on
$[-\sqrt d,\sqrt d]$ and their midpoints as reconstruction levels. The
resulting Cartesian-product codebook contains $n^{d_\text{eff}}\leq K$
points, and every coordinate incurs squared error at most $d/n^2$.
Consequently,
\begin{align*}
    \Exs_{Z\sim\widetilde{\Prob}}
    \big[\vecnorm{Z-\quantizer{\tilde\phi}(Z)}{2}^2\big]
    \leq \frac{d\,d_\text{eff}}{n^2}
    \leq \frac{4d\,d_\text{eff}}{K^{2/d_\text{eff}}}.
\end{align*}
Mapping every reconstruction point $\tilde e$ back to $U\tilde e\in L$
preserves Euclidean distances. Thus the same distortion is achievable by
a VQ codebook in $\real^d$, proving
\begin{align*}
    \optimalErrVQ(\Prob,K)
    \leq \frac{4d\,d_\text{eff}}{K^{2/d_\text{eff}}}.
\end{align*}

We next prove the PQ and SQ lower bounds. Consider
\begin{align*}
    L=\{\alpha\bm 1:\alpha\in\real\},
    \qquad
    \bm 1=(1,\ldots,1)^\top\in\real^d,
\end{align*}
and let $X=A\bm 1$, where $A\sim\operatorname{Unif}[-1,1]$.
Then $X$ is supported on $L\cap[-1,1]^d$.

We use the following elementary one-dimensional fact. If
$A\sim\operatorname{Unif}[-1,1]$, then every quantizer with at most $n$
reconstruction values satisfies
\begin{align}
    \Exs\big[(A-\widehat A)^2\big]\geq \frac{1}{3n^2}.
    \label{eq:uniform-scalar-lower-bound}
\end{align}
Indeed, for a nearest-neighbor scalar quantizer, let the lengths of its
nonempty Voronoi cells be $\ell_1,\ldots,\ell_r$, where $r\leq n$ and
$\sum_{j=1}^r\ell_j=2$. Replacing the reconstruction value in each cell
by its midpoint can only decrease squared error, so
\begin{align*}
    \Exs\big[(A-\widehat A)^2\big]
    \geq \frac12\sum_{j=1}^r\frac{\ell_j^3}{12}
    =\frac1{24}\sum_{j=1}^r\ell_j^3
    \geq\frac1{24}\frac{(\sum_{j=1}^r\ell_j)^3}{r^2}
    \geq\frac{1}{3n^2},
\end{align*}
where the penultimate inequality follows from convexity.

Now consider an arbitrary coordinate-block PQ quantizer with $M$ blocks
of dimensions $d_1,\ldots,d_M$, where $d_m\geq1$ and
$\sum_{m=1}^M d_m=d$. Let the $m$-th sub-codebook contain $n_m$ points,
with $\prod_{m=1}^M n_m\leq K$. The input to block $m$ is
$A\bm 1_{d_m}$. For any reconstruction point in $\real^{d_m}$,
orthogonally projecting it onto the line
$\operatorname{span}(\bm 1_{d_m})$ cannot increase its distance to
$A\bm 1_{d_m}$. Hence an $n_m$-point vector quantizer for this block
cannot have smaller distortion than the optimal $n_m$-level scalar
quantizer for $A$, scaled by $\vecnorm{\bm 1_{d_m}}{2}^2=d_m$.
Using \eqref{eq:uniform-scalar-lower-bound},
\begin{align*}
    \Exs\big[
    \vecnorm{A\bm 1_{d_m}-\quantizer{\phi_m}(A\bm 1_{d_m})}{2}^2
    \big]
    \geq \frac{d_m}{3n_m^2}
    \geq \frac{d_m}{4n_m^2}.
\end{align*}
Summing over the mutually orthogonal coordinate blocks gives
\begin{align*}
    \Exs\big[\vecnorm{X-\quantizer{\phi}(X)}{2}^2\big]
    &\geq \frac14\sum_{m=1}^M\frac{d_m}{n_m^2}\\
    &\geq \frac{M}{4}
    \left(\prod_{m=1}^M\frac{d_m}{n_m^2}\right)^{1/M}\\
    &\geq \frac{M}{4}
    \left(\prod_{m=1}^M n_m\right)^{-2/M}\\
    &\geq \frac{M}{4}K^{-2/M}.
\end{align*}
Here the second line is AM--GM, and the third uses $d_m\geq1$.
Because the PQ codebooks and rate allocation $(n_1,\ldots,n_M)$ were
arbitrary, taking the infimum proves
\begin{align*}
    \optimalErrPQ(\Prob,K,M)\geq\frac{M}{4}K^{-2/M},
    \qquad 2\leq M\leq d.
\end{align*}
Finally, SQ is the case $M=d$ and $d_m=1$ for every block, yielding
\begin{align*}
    \optimalErrSQ(\Prob,K)\geq\frac{d}{4}K^{-2/d}.
\end{align*}
This completes the proof.

\section{Relationship to Classical Quantization Theory}
\label{appendix:classical-quantization-theory}

The $K^{-2/d}$ distortion scaling in Proposition~\ref{thm:worst-case} has the same high-rate exponent as Zador's foundational asymptotic quantization theorem~\citep{Zador1982} and subsequent multidimensional extensions~\citep{BucklewWise1982}, as surveyed by \citet{GrayNeuhoff1998}. These classical results characterize optimal quantization distortion for regular $d$-dimensional sources in the high-rate regime. Proposition~\ref{thm:worst-case}, by contrast, is a finite-$K$ worst-case result over distributions supported on a bounded domain. Similarly, under the quantizer families defined in \Cref{sec:quantization error analysis}, the nested ordering $\optimalErrVQ \le \optimalErrPQ \le \optimalErrSQ$ follows directly from the structural inclusion
$\mathrm{SQ} \subseteq \mathrm{PQ} \subseteq \mathrm{VQ}$: SQ is the fully factorized special case of PQ, while a PQ codebook is a Cartesian-product-structured VQ codebook. This relationship is standard in classical treatments of vector and product quantization~\citep{GershoGray1992,GrayNeuhoff1998}. Proposition~\ref{thm:worst-case} therefore serves as a self-contained finite-rate reference point within our framework rather than as a novel asymptotic quantization result.

The main theoretical distinction of our analysis lies in Proposition~\ref{thm:adaptive}, which provides the explicit, non-asymptotic lower bound $\optimalErrPQ(P,K,M) \ge \tfrac{M}{4}K^{-2/M}$
for the fixed coordinate-block PQ family on an axis-misaligned one-dimensional source. This yields a quantitative VQ--PQ separation whose scaling depends explicitly on the number of blocks $M$. Prior work has established that PQ distortion can depend strongly on the choice of subspace decomposition and can be improved by optimizing the orientation or decomposition of the feature space in approximate nearest-neighbor search~\citep{GeEtAl2013,NorouziFleet2013}. To our knowledge, however, this literature does not provide the explicit finite-$K$, block-count-dependent lower bound established here.

Our framework also extends beyond classical distortion analysis in two directions. Appendix~\ref{appendix:gradient gap distortion} connects
quantization distortion to the STE-induced gradient gap, relating quantization error to the optimization dynamics of learned discrete representations. In addition, our empirical study compares VQ, PQ, and SQ under matched source representations, token counts, and composite
code-space cardinalities, providing a controlled visual-tokenization setting for examining the theoretical relationships above.

\section{Distortion Minimization and STE Gradient Discrepancy}
\label{appendix:gradient gap distortion}

In this section, we derive a local relationship between quantization distortion and the gradient discrepancy induced by the Straight-Through Estimator
(STE)~\citep{Bengio2013EstimatingOP}. The analysis shows that, under standard local smoothness, the gradient discrepancy vanishes quadratically with the quantization distance, providing theoretical support for distortion minimization as a means of improving the local fidelity of the STE gradient.

\textbf{Problem Setup}: Let $\mathcal{L}$ denote the loss as a function of the latent representation. In the quantization process, the encoder outputs a continuous latent vector $z_e$, which is discretized into a quantized latent representation $z_q$. During backpropagation, the non-differentiable quantization
operation is bypassed using the STE, which propagates the gradient evaluated at the quantized representation:
\begin{equation}
    \text{STE Approximation:} \quad
    \frac{\partial \mathcal{L}}{\partial z_e}
    \leftarrow
    \frac{\partial \mathcal{L}}{\partial z_q}.
\vspace{-3pt}
\end{equation}

We define the \textbf{gradient gap} $\mathcal{G}$ as the squared discrepancy between the gradients evaluated at the continuous and quantized latent representations:
\begin{equation}
    \mathcal{G}
    \triangleq
    \left\|
    \frac{\partial \mathcal{L}}{\partial z_e}
    -
    \frac{\partial \mathcal{L}}{\partial z_q}
    \right\|_2^2.
\vspace{-3pt}
\end{equation}

\textbf{Local Taylor Analysis}:
Assume that $\mathcal{L}$ is twice differentiable in a neighborhood of $z_q$.
Let
\[
    \mathbf{H}(z_q)
    =
    \left.
    \frac{\partial^2 \mathcal{L}}{\partial x^2}
    \right|_{x=z_q}
\]
denote the Hessian at $z_q$. A first-order expansion of the gradient around $z_q$ gives
\begin{equation}
    \frac{\partial \mathcal{L}}{\partial z_e}
    =
    \frac{\partial \mathcal{L}}{\partial z_q}
    +
    \mathbf{H}(z_q)(z_e-z_q)
    +
    o(\|z_e-z_q\|_2).
    \label{eq:taylor}
\vspace{-3pt}
\end{equation}

Therefore,
\begin{equation}
    \frac{\partial \mathcal{L}}{\partial z_e}
    -
    \frac{\partial \mathcal{L}}{\partial z_q}
    =
    \mathbf{H}(z_q)(z_e-z_q)
    +
    o(\|z_e-z_q\|_2),
\vspace{-3pt}
\end{equation}
and hence
\begin{equation}
    \mathcal{G}
    =
    \left\|
    \mathbf{H}(z_q)(z_e-z_q)
    +
    o(\|z_e-z_q\|_2)
    \right\|_2^2
    =
    \mathcal{O}(\|z_e-z_q\|_2^2).
    \label{eq:bound}
\vspace{-3pt}
\end{equation}

Moreover, the leading-order term is
\begin{equation}
    \mathcal{G}
    =
    \|\mathbf{H}(z_q)(z_e-z_q)\|_2^2
    +
    o(\|z_e-z_q\|_2^2).
\vspace{-3pt}
\end{equation}

Since $\|z_e-z_q\|_2^2$ is exactly the squared quantization distortion, Equation~\ref{eq:bound} establishes a local connection between distortion
and the STE-induced gradient discrepancy: as the quantization distortion approaches zero, the gradient gap also vanishes at least quadratically in the
quantization distance. Thus, reducing distortion makes the gradient evaluated at the quantized representation locally closer to that evaluated at the
continuous representation. This provides a theoretical explanation for why distortion reduction can improve the fidelity of gradient propagation through the STE.

\section{Effect of Latent Scale on Quantization Distortion}
\label{appendix:error-variance-relationship}

We provide both a theoretical analysis and an empirical verification of the effect of latent feature scale on optimal quantization distortion. The result motivates the matched-distribution condition introduced in \Cref{sec:fair comparison}.

\paragraph{Theoretical analysis.}
Let the optimal $K$-point vector-quantization distortion under squared Euclidean distance be
\begin{equation}
    \mathcal{E}^{\star}(X,K)
    =
    \inf_{\mathcal C:\,|\mathcal C|=K}
    \mathbb{E}
    \left[
        \min_{c\in\mathcal C}
        \|X-c\|_2^2
    \right].
    \label{eq:optimal_vq_distortion}
\end{equation}

\begin{proposition}[Scale dependence of optimal distortion]
\label{prop:distortion-scaling}
For any scalar $a\in\mathbb{R}$,
\begin{equation}
    \mathcal{E}^{\star}(aX,K)
    =
    a^2\mathcal{E}^{\star}(X,K).
    \label{eq:appendix_distortion_scaling}
\end{equation}
\end{proposition}

\begin{proof}
For $a\neq0$, the mapping
\[
    \mathcal C'
    \mapsto
    a\mathcal C'
    =
    \{ac' : c'\in\mathcal C'\}
\]
is a bijection over the set of all $K$-point codebooks. Hence,
\begin{align}
\mathcal{E}^{\star}(aX,K)
&=
\inf_{\mathcal C:\,|\mathcal C|=K}
\mathbb{E}
\left[
    \min_{c\in\mathcal C}
    \|aX-c\|_2^2
\right]
\\
&=
\inf_{\mathcal C':\,|\mathcal C'|=K}
\mathbb{E}
\left[
    \min_{c'\in\mathcal C'}
    \|aX-ac'\|_2^2
\right]
\\
&=
a^2
\inf_{\mathcal C':\,|\mathcal C'|=K}
\mathbb{E}
\left[
    \min_{c'\in\mathcal C'}
    \|X-c'\|_2^2
\right]
\\
&=
a^2\mathcal{E}^{\star}(X,K).
\end{align}
For $a=0$, both sides are zero, completing the proof.
\end{proof}

Proposition~\ref{prop:distortion-scaling} shows that the optimal squared quantization distortion is homogeneous of degree two with respect to global scaling of the source distribution. Importantly, this result does not rely on a Gaussian assumption or a high-rate approximation.

A useful special case follows by writing a location-scale family as
\begin{equation}
    X = \mu + \sigma Z,
\end{equation}
where $Z$ has a fixed reference distribution. Since translation can be absorbed by an equal translation of the codebook,
\begin{equation}
    \mathcal{E}^{\star}(\mu+\sigma Z,K)
    =
    \sigma^2\mathcal{E}^{\star}(Z,K).
    \label{eq:variance_distortion_relation}
\end{equation}
Thus, when the shape of the latent distribution is fixed and only its scale varies, the optimal distortion grows linearly with the source variance $\sigma^2$.

\section{Limitations}
\label{app:limitations}
Our results should be interpreted within several important scope boundaries. First, we study the intrinsic effectiveness of quantization under a nominal fixed-length representation rate, $R=T\log_2 K$, rather than the expected bitstream length produced by entropy coding. Consequently, our analysis does
not directly characterize entropy-constrained quantization, variable-length coding, or adaptive bit allocation, where the probability distribution of discrete symbols also contributes to the effective coding rate. Extending the framework to these settings would provide a closer connection to classical
rate-distortion optimization in learned compression.

Second, our theoretical analysis adopts squared Euclidean quantization distortion as the intrinsic information-loss measure. This choice enables a clean characterization of source scaling and of the structural relationships
among VQ, PQ, and SQ, but it does not directly cover perceptual, semantic, or task-dependent distortion measures. Moreover, the theoretical distortion hierarchy concerns the optimal achievable distortion within each quantizer family; it does not imply that every practical VQ optimization algorithm will outperform every PQ or SQ method under finite-data and finite-optimization regimes.

Finally, our conclusions concern \emph{intrinsic quantization and reconstruction effectiveness}, rather than downstream generative modeling performance. Improved reconstruction fidelity does not necessarily imply improved generation
quality~\citep{Ramanujan2024WhenWorse,Yao2025ReconstructionVG}, since downstream performance also depends on properties of the resulting discrete representation and on the capacity and optimization of the generative model. Understanding how intrinsic quantization distortion interacts with representation
modelability and downstream generation therefore remains an important direction for future work.

\section{Experimental Details}
\label{appendix:experimental details}
\paragraph{Data Preprocessing.} For all three datasets, ImageNet-1k~\citep{Deng2009ImageNetAL}, FFHQ~\citep{Karras2018ASG}, and CelebA-HQ~\citep{Karras2017ProgressiveGO}, we follow Llama Gen~\citep{Sun2024AutoregressiveMB} by applying iterative box downsampling to resize all images to a 256×256 resolution. 

\paragraph{Encoder-Decoder Architecture.} For experiments conducted in the latent space, we adopt the VQ-Transplant framework~\citep{Fang2026VQTransplant} and initialize all quantization methods with the same pretrained VAR tokenizer~\citep{Tian2024VisualAM}. This setup fixes the encoder-decoder architecture and latent representation across methods, enabling a controlled comparison of different quantization algorithms under the same latent space. The encoder is a U-Net~\citep{Ronneberger2015UNetCN} that downsamples input images by a factor of 16, producing latent features $\bm{z}_e = \mathcal{E}_\theta(\bm{x}) \in \mathbb{R}^{16 \times 16 \times 32}$ with a spatial resolution of $16 \times 16$. In addition to these latent-space experiments, we also consider a pixel-space setting constructed directly from image pixels without relying on a pretrained encoder-decoder representation.

\paragraph{Training Details.}
Following~\citep{Fang2026VQTransplant}, all latent-space models were trained on two NVIDIA H100 GPUs using the AdamW optimizer~\citep{Loshchilov2017DecoupledWD} with $\beta_1=0.9$ and $\beta_2=0.95$. During VQ module substitution, we used an initial learning rate of $10^{-4}$ with linear decay to $10^{-5}$, while the learning rate was fixed at $10^{-5}$ during decoder adaptation. For ImageNet-1k, VQ module substitution and decoder adaptation were performed for 2 and 5 epochs, respectively. For both FFHQ and CelebA-HQ, VQ module substitution and decoder adaptation were each performed for 30 epochs. The batch size was fixed at 32 per GPU for all latent-space experiments. We additionally conducted pixel-space experiments on CelebA-HQ using a single-stage training procedure for 30 epochs. We used the same AdamW optimizer and batch size as in the latent-space experiments, with an initial learning rate of $10^{-4}$ linearly decayed to $10^{-5}$.

\paragraph{Loss Weights.} For the latent-space experiments, we follow the loss formulation and notation of VQ-Transplant~\citep{Fang2026VQTransplant}. Here, $\lambda_P$ and $\lambda_G$ denote the perceptual and adversarial loss weights, respectively, while $\gamma$ controls the weight of the distribution-matching objective. We set $\lambda_P=1$ and $\lambda_G=0.4$. For configurations using the Wasserstein distance (Wasserstein VQ and Wasserstein VP2), we set $\gamma=0.2$, while for configurations using the MMD distance (MMD VQ and MMD VP2), we set $\gamma=0.5$. For the pixel-space experiments, the training objective consists of the reconstruction, perceptual, and quantization losses, with the perceptual loss weight fixed to $\lambda_P=1$. For Wasserstein- and MMD-based quantizers, the corresponding distribution-matching weight $\gamma$ follows the same setting as in the latent-space experiments.

\begin{wrapfigure}[6]{r}{0.53\textwidth}
\small
\vspace{-7mm}
\begin{center}
\includegraphics[width=0.50\textwidth]{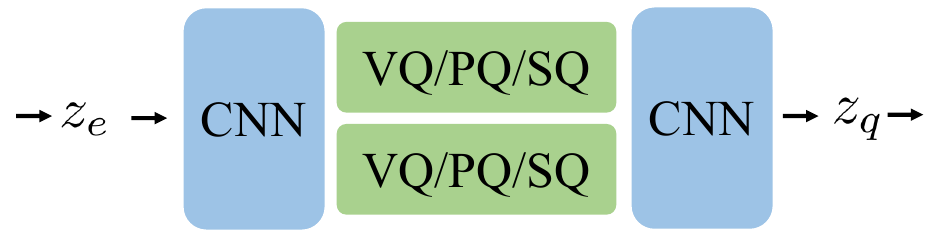}
\vspace{-2mm}
\caption{\small{Illustration of the implementation details.}}
\label{fig:vq-pq-sq}
\end{center}
\end{wrapfigure} 
\paragraph{Latent-Space VQ Implementation.} We implement five VQ variants within the VQ-Transplant framework: Vanilla VQ~\citep{Oord2017NeuralDR}, EMA VQ~\citep{Razavi2019GeneratingDH}, Online VQ~\citep{Zheng2023OnlineCC}, Wasserstein VQ~\citep{Fang2026distributional}, and MMD VQ~\citep{Fang2026VQTransplant}. We first use the pretrained VAR encoder to extract latent features $\bm{z}_e = \mathcal{E}_\theta(\bm{x}) \in \mathbb{R}^{16 \times 16 \times 32}$. A three-layer convolutional projector then processes $\bm{z}_e$ while preserving the 32-dimensional feature space. Each 32-dimensional feature vector is partitioned into two 16-dimensional vectors, each of which is treated as an independent quantization token and processed by a separate VQ module. Consequently, the $16\times16$ latent grid yields a total of $T=512$ quantization tokens. The two quantized 16-dimensional vectors at each spatial location are subsequently concatenated to recover a 32-dimensional representation, as illustrated in Figure~\ref{fig:vq-pq-sq}. A second three-layer convolutional projector maps the concatenated quantized features to the decoder input. The code-space cardinality of each token is fixed to $K=65{,}536$, yielding a nominal representation rate of $R=T\log_2 K$ with $T=512$.

\paragraph{Latent-Space PQ Implementation.}
We implement five PQ variants within the VQ-Transplant framework: Vanilla VP2~\citep{Oord2017NeuralDR}, EMA VP2~\citep{Razavi2019GeneratingDH}, Online VP2~\citep{Zheng2023OnlineCC}, Wasserstein VP2~\citep{Fang2026distributional}, and MMD VP2~\citep{Fang2026VQTransplant}. The overall architecture and tokenization scheme are identical to those used for VQ, with the only change being the internal quantization structure. Specifically, each 16-dimensional quantization token is further partitioned into $M=2$ 8-dimensional sub-vectors. The two sub-vectors are independently quantized using separate codebooks of size $256$, producing a composite code-space cardinality of $256^2=65{,}536$ for each token. Therefore, PQ retains the same number of quantization tokens, $T=512$, and the same effective code-space cardinality, $K=65{,}536$, as VQ, ensuring a matched nominal representation rate.

\paragraph{Latent-Space SQ Implementation.}
We implement three SQ variants within the VQ-Transplant framework: FSQ~\citep{Mentzer2023FiniteSQ}, LFQ~\citep{yu2024language}, and BSQ~\citep{zhao2024image}. The surrounding architecture is kept consistent with the VQ and PQ implementations, while the quantization module is replaced by the corresponding scalar quantizer. For LFQ, each 16-dimensional quantization token is discretized using binary scalar levels, resulting in a code-space cardinality of $2^{16}=65{,}536$. BSQ uses the same 16-dimensional token configuration, with normalization applied before binary spherical quantization, likewise yielding $2^{16}=65{,}536$ possible codes. For FSQ, a three-layer convolutional projector reduces the 32-dimensional feature representation to 16 dimensions, which are organized into two 8-dimensional quantization tokens. Each scalar dimension is discretized into four fixed levels, giving a code-space cardinality of $4^8=65{,}536$ per token. A symmetric three-layer convolutional projector then maps the quantized representation back to 32 dimensions. Thus, all VQ, PQ, and SQ configurations use $T=512$ quantization tokens with an effective code-space cardinality of $K=65{,}536$ per token, ensuring matched nominal representation rate across methods.

\paragraph{Pixel-Space Implementation.}
In addition to the latent-space experiments, we conduct pixel-space experiments on CelebA-HQ, where quantization is performed without a pretrained encoder-decoder. Given an input image $\bm{x}\in\mathbb{R}^{256\times256\times3}$, we first apply PixelUnshuffle with a factor of $4$, which deterministically rearranges each $4\times4$ spatial block into the channel dimension and produces a $64\times64\times48$ representation without discarding information. This representation is then processed by a three-layer convolutional projector that maps the 48 channels to the quantization feature dimension. The projected features are quantized using the corresponding quantization module. Since the resulting feature grid has a spatial resolution of $64\times64$, the pixel-space setting contains $T=4096$ quantization tokens. The code-space cardinality of each token is fixed to $K=65{,}536$, yielding a nominal representation rate of $R=T\log_2K$. After quantization, a symmetric three-layer convolutional projector maps the quantized features back to 48 channels, and PixelShuffle with a factor of $4$ is applied to reconstruct the image at the original $256\times256$ resolution. All compared methods in the pixel-space setting share the same PixelUnshuffle/PixelShuffle operations, convolutional projectors, token count, and code-space cardinality, such that the quantization module remains the primary varying component across methods.

\paragraph{Remark.}
Our comparisons are controlled separately within the latent-space and pixel-space settings. In the latent-space experiments, all methods operate on the same pretrained VAR encoder output and use $T=512$ quantization tokens with a code-space cardinality of $K=65{,}536$ per token. In the pixel-space experiments, all methods share the same PixelUnshuffle operation and convolutional projector, and use $T=4096$ quantization tokens with the same code-space cardinality $K=65{,}536$. Therefore, the nominal representation rate $R=T\log_2K$ is matched across quantization methods within each setting, satisfying Condition~\ref{condition:rate constraints}. With respect to Condition~\ref{condition:feature distribution}, all compared methods within each setting receive the same source representation before quantizer-specific processing, providing a controlled basis for comparing intrinsic quantization effectiveness.

\begin{table*}[!h]
\vspace*{-6pt}
\centering
\caption{\small{Quantization and reconstruction performance under varying per-token code-space cardinality $K$ with the token count fixed at $T=512$.}}
\vspace*{-6pt}
\label{tab:codebook size}
\resizebox{\textwidth}{!}{%
\begin{tabular}{lcccccccccc} \hline
Methods & Phase & Tokens & Codebook Size $K$ & $\mathcal{E}$($\downarrow$) & $\mathrm{U}$ ($\uparrow$) & PSNR($\uparrow$) & SSIM($\uparrow$) & LPIPS ($\downarrow$) & r-FID($\downarrow$) & r-IS($\uparrow$) \\\hline
MMD VQ & Substitution & 512 & 1024 & 0.318 & 99.6\% & 23.75 & 60.8 & 0.162 & 2.68 & 167.6 \\
MMD VQ & Substitution & 512 & 2048 & 0.296 & 99.4\% & 24.12 & 62.4 & 0.159 & 2.59 & 168.8 \\
MMD VQ & Substitution & 512 & 4096 & 0.273 & 99.4\% & 24.41 & 63.5 & 0.141 & 1.96 & 178.0 \\
MMD VQ & Substitution & 512 & 8192 & 0.252 & 99.5\% & 24.67 & 64.6 & 0.135 & 1.85 & 181.0 \\
MMD VQ & Substitution & 512 & 16384 & 0.234 & 99.8\% & 24.89 & 65.4 & 0.130 & 1.84 & 183.7 \\
MMD VQ & Substitution & 512 & 32768 & 0.215 & 99.7\% & 25.06 & 66.3 & 0.126 & 1.79 & 184.8\\
MMD VQ & Substitution & 512 & 65536 & 0.201 & 99.9\% & 25.24 & 66.8 & 0.121 & 1.69 & 187.3 \\\hline\hline 
MMD VQ & Adaptation & 512 & 1024 & 0.318 & 99.6\% & 23.06 & 58.8 & 0.148 & 1.90 & 169.8\\
MMD VQ & Adaptation & 512 & 2048 & 0.296 & 99.4\% & 23.58 & 61.2 & 0.137 & 1.63 & 176.6\\
MMD VQ & Adaptation & 512 & 4096 & 0.273 & 99.4\% & 23.89 & 61.8 & 0.128 & 1.28 &  185.1\\
MMD VQ & Adaptation & 512 & 8192 & 0.252 & 99.5\% & 24.11 & 62.9 & 0.121 & 1.18 & 187.9\\
MMD VQ & Adaptation & 512 & 16384 & 0.234 & 99.8\% & 24.31 & 63.7 & 0.115 & 1.05 & 191.2\\
MMD VQ & Adaptation & 512 & 32768 & 0.215 & 99.9\% & 24.53 & 64.7 & 0.110 & 0.97 & 194.1\\
MMD VQ & Adaptation & 512 & 65536 & 0.201 & 99.9\% & 24.65 & 65.0 & 0.106 & 0.86 & 197.1\\\hline
\end{tabular}
}
\vspace*{-8pt}
\end{table*}

\begin{table*}[!h]
\vspace*{-2pt}
\centering
\caption{\small{Quantization and reconstruction performance under varying token count $T$ with the per-token code-space cardinality fixed at $K=16{,}384$.}}
\vspace*{-6pt}
\label{tab:token length}
\resizebox{\textwidth}{!}{%
\begin{tabular}{lcccccccccc} \hline
Methods & Phase & Tokens & Codebook Size $K$ & $\mathcal{E}$($\downarrow$) & $\mathrm{U}$ ($\uparrow$) & PSNR($\uparrow$) & SSIM($\uparrow$) & LPIPS ($\downarrow$) & r-FID($\downarrow$) & r-IS($\uparrow$) \\\hline
MMD VQ & Substitution & 256 & 16384 & 0.369 & 99.6\% & 22.97 & 57.2 & 0.194 & 4.91 & 141.4 \\
MMD VQ & Substitution & 512 & 16384 & 0.234 & 99.8\% & 24.89 & 65.4 & 0.130 & 1.84 & 183.7 \\
MMD VQ & Substitution & 1024 & 16384 & 0.098 & 100\% & 26.40 & 71.0 & 0.100 & 2.01 & 191.7 \\
MMD VQ & Substitution & 2048 & 16384 & 0.035 & 100\% & 27.16 & 73.1 & 0.089 & 2.36 & 192.2 \\\hline\hline 
MMD VQ & Adaptation & 256 & 16384 & 0.369 & 99.6\% & 22.41 & 55.9 & 0.171 & 3.06 & 148.9 \\
MMD VQ & Adaptation & 512 & 16384 & 0.234 & 99.8\% & 24.31 & 63.7 & 0.115 & 1.05 & 191.2\\
MMD VQ & Adaptation & 1024 & 16384 & 0.098 & 100\% & 26.03 & 69.6 & 0.079 & 0.54 & 210.1\\
MMD VQ & Adaptation & 2048 & 16384 & 0.035 & 100\% & 27.31 & 73.2 & 0.060 & 0.42 & 217.0 \\\hline
\end{tabular}
}
\vspace*{-6pt}
\end{table*}

\begin{table*}[!h]
\centering
\vspace*{-2pt}
\caption{\small{Comparison of different token-count and code-space allocations under matched nominal rate.}}
\vspace*{-6pt}
\label{tab:matched rate}
\resizebox{\textwidth}{!}{%
\begin{tabular}{lccccccccccc}
\toprule
Methods & Phase & Tokens & Codebook Size $K$ & Rate $R$ & $\mathcal{E}$($\downarrow$) & $\mathrm{U}$ ($\uparrow$) & PSNR($\uparrow$) & SSIM($\uparrow$) & LPIPS ($\downarrow$) & r-FID($\downarrow$) & r-IS($\uparrow$) \\
\midrule
MMD VQ & Substitution & 512 &  16384 & 7,168 bits & 0.234 & 99.8\% & 24.89 & 65.4 & 0.130 & 1.84 & 183.7 \\
MMD VQ & Adaptation & 512 & 16384 & 7,168 bits & 0.234 & 99.8\% & 24.31 & 63.7 & 0.115 & 1.05 & 191.2\\
MMD VQ & Substitution & 1024 & 128 & 7,168 bits & 0.243 & 100\% & 24.72 & 64.9 & 0.132 & 1.80 & 184.5 \\
MMD VQ & Adaptation & 1024 & 128 & 7,168 bits & 0.243 & 100\% & 24.01 & 63.0 & 0.118 & 1.10 & 190.4 \\
\midrule
\midrule
MMD VQ & Substitution & 512 & 65536 & 8,192 bits & 0.201 & 99.9\% & 25.24 & 66.8 & 0.121 & 1.69 & 187.3 \\
MMD VQ & Adaptation & 512 & 65536 & 8,192 bits & 0.201 & 99.9\% & 24.65 & 65.0 & 0.106 & 0.86 & 197.1\\
MMD VQ & Substitution & 1024 & 256 & 8,192 bits & 0.212 & 100\% & 25.11 & 66.4 & 0.124 & 1.67 & 187.6 \\
MMD VQ & Adaptation & 1024 & 256 & 8,192 bits & 0.212 & 100\% & 24.51 & 64.6 & 0.108 & 0.96 & 195.3 \\
\bottomrule
\end{tabular}
}
\vspace*{-12pt}
\end{table*}

\section{Rate--Distortion Analysis of Code-Space Cardinality and Token Count}
\label{appendix:codebook size and tokens}

We first examine the effect of code-space cardinality while fixing the token count at $T=512$. Specifically, we vary the per-token cardinality from $K=1{,}024$ to $K=65{,}536$ by successive factors of two. Under the nominal fixed-length rate $R=T\log_2K$, increasing $K$ increases the representation budget. As shown in Table~\ref{tab:codebook size}, the quantization distortion $\mathcal{E}$ decreases monotonically from $0.318$ to $0.201$, while r-FID after decoder adaptation improves from $1.90$ to $0.86$. Thus, increasing $K$ at fixed $T$ reduces quantization distortion and improves reconstruction quality.

We next vary the token count from $T=256$ to $T=2{,}048$ while fixing $K=16{,}384$. As shown in Table~\ref{tab:token length}, increasing $T$ reduces distortion from $0.369$ to $0.035$, while adapted r-FID improves from $3.06$ to $0.42$. These results show that increasing token count can likewise substantially reduce information loss. The different improvement patterns obtained by varying $T$ and $K$ further indicate that the two quantities may affect representation structure differently despite both contributing to the nominal rate.

Finally, we compare different allocations of the same nominal representation rate. Under $R=T\log_2K$, doubling the token count and squaring the
per-token code-space cardinality produce the same nominal rate:
\begin{equation}
2T\log_2K = T\log_2K^2,
\end{equation}
This identity concerns only the representation budget and does not imply identical distortion. Table~\ref{tab:matched rate} compares matched-rate configurations, including $(T,K)=(512,16{,}384)$ versus $(1{,}024,128)$ and $(512,65{,}536)$ versus $(1{,}024,256)$. These pairs achieve similar, though not identical, distortion and reconstruction performance, consistent with the rate accounting in \Cref{sec:rate_distortion_definition} and with the fact that different rate allocations can induce different representation structures.


\begin{table*}[!h]
\centering
\caption{\small{Quantization and reconstruction performance of VQ, PQ, and SQ on FFHQ under matched token count ($T=512$) and code-space cardinality ($K=65{,}536$). Within each quantization type and phase (Substitution/Adaptation), optimal values are \underline{underlined}; overall best results per metric are \textbf{bold}.}}
\vspace*{-8pt}
\label{tab:var-ffhq}
\resizebox{0.95\textwidth}{!}{%
\begin{tabular}{lcccccccccc}
\toprule
Approaches  & Types & Phase & Tokens & $K$ & $\mathcal{E}$($\downarrow$) & $\mathrm{U}$ ($\uparrow$)  & PSNR($\uparrow$) & SSIM($\uparrow$) & LPIPS ($\downarrow$) & r-FID($\downarrow$) \\
\midrule
Vanilla VQ & VQ & Substitution & 512 & 65536 & 0.286 & 0.2\% & 24.20 & 66.7 & 0.159 & 7.16\\
EMA VQ & VQ & Substitution & 512 & 65536 & 0.146 & 52.9\% & 27.66 & 77.4 & 0.075 & 2.13\\
Online VQ & VQ & Substitution & 512 & 65536 & 0.186 & 12.4\% & 26.98 & 75.2 & 0.088 & 2.41\\
Wasserstein VQ & VQ & Substitution & 512 & 65536 & 0.132 & \underline{99.5\%} & 28.05 & \textbf{\underline{78.3}} & 0.070 & \underline{2.13}\\
MMD VQ & VQ & Substitution & 512 & 65536 & \textbf{\underline{0.131}} & 99.3\% & \textbf{\underline{28.06}} & \textbf{\underline{78.3}} & \textbf{\underline{0.069}} & 2.14\\
\midrule
Vanilla VP2 & PQ & Substitution & 512 & 65536 & 0.159 & 40.3\% & 27.33 & 76.7 & 0.078 & 2.38\\
EMA VP2 & PQ & Substitution & 512 & 65536 & \underline{0.136} & \textbf{\underline{100\%}} & \underline{27.94} & \underline{77.9} & \underline{0.071} & \underline{2.14}\\
Online VP2 & PQ & Substitution & 512 & 65536 & 0.137 & \textbf{\underline{100\%}} & 27.89 & \underline{77.9} & 0.072 & 2.43\\
Wasserstein VP2 & PQ & Substitution & 512 & 65536 & 0.146 & 99.9\% & 27.64 & 77.6 & 0.079 & 2.95\\
MMD VP2 & PQ & Substitution & 512 & 65536 & 0.139 & \textbf{\underline{100\%}} & 27.83 & \underline{77.9} &  0.072 & 2.47\\
\midrule
FSQ & SQ & Substitution & 512 & 65536 & 0.170 & \textbf{\underline{100\%}} & \underline{27.27} & 75.4 & 0.081 & 3.18\\
LFQ & SQ & Substitution & 512 & 65536 & 0.224 & 7.2\% & 25.62 & 71.6 & 0.110 & 3.70\\
BSQ & SQ & Substitution & 512 & 65536 & \underline{0.155} & \textbf{\underline{100\%}} & 27.04 & \underline{76.7} & \underline{0.080}  & \textbf{\underline{2.09}}\\
\midrule
\midrule
Vanilla VQ & VQ & Adaptation & 512 & 65536 &  0.286 & 0.2\%  & 23.63 & 63.8 & 0.142 & 2.31 \\
EMA VQ & VQ & Adaptation & 512 & 65536 &  0.146 & 52.9\%  & 27.35 & 75.9 & 0.071 &  1.25 \\
Online VQ & VQ & Adaptation & 512 & 65536 & 0.186 & 12.4\%  & 26.57 & 73.3 & 0.085 &  1.77 \\
Wasserstein VQ & VQ & Adaptation & 512 & 65536 &  0.132 & \underline{99.5\%}  & 27.69 & 76.5 & \textbf{\underline{0.067}} &  0.89\\
MMD VQ & VQ & Adaptation & 512 & 65536 & \textbf{\underline{0.131}} & 99.3\%  &  \textbf{\underline{27.75}} & \textbf{\underline{76.9}} & \textbf{\underline{0.067}} &  \textbf{\underline{0.85}}\\
\midrule
Vanilla VP2 & PQ & Adaptation & 512 & 65536 & 0.159 & 40.3\%  & 27.11 & 75.2 & 0.074 & 1.33\\
EMA VP2 & PQ & Adaptation & 512 & 65536 &  \underline{0.136} & \textbf{\underline{100\%}}  & 27.57 & 76.1 & 0.070 & \underline{1.05}\\
Online VP2 & PQ & Adaptation & 512 & 65536 &  0.137 & \textbf{\underline{100\%}}  & 27.54 & 75.8 & 0.069 &  1.10\\
Wasserstein VP2 & PQ & Adaptation & 512 & 65536 & 0.146 & 99.9\% & 27.28 & 75.5 & 0.072 & 1.11 \\
MMD VP2 & PQ & Adaptation & 512 & 65536 & 0.139 & \textbf{\underline{100\%}}  & \underline{27.58}  & \underline{76.4} & \underline{0.068} & 1.16 \\
\midrule
FSQ & SQ & Adaptation & 512 & 65536 &  0.170 & \textbf{\underline{100\%}}  & \underline{26.83} & 73.5 & \underline{0.076} & 1.61 \\
LFQ & SQ & Adaptation & 512 & 65536 &  0.224 & 7.2\%  & 25.23 & 69.9 & 0.102 & 1.74 \\
BSQ & SQ & Adaptation & 512 & 65536 &  \underline{0.155} & \textbf{\underline{100\%}}  & 26.73 & \underline{75.1} & 0.078 & \underline{1.54} \\
\bottomrule
\end{tabular}
}
\vspace{-12pt}
\end{table*}

\begin{table*}[!h]
\centering
\caption{\small{Quantization and reconstruction performance of VQ, PQ, and SQ on CelebA-HQ under matched token count ($T=512$) and code-space cardinality ($K=65{,}536$). Within each quantization type and phase (Substitution/Adaptation), optimal values are \underline{underlined}; overall best results per metric are \textbf{bold}.}}
\vspace*{-5pt}
\label{tab:var-celebahq}
\resizebox{0.95\textwidth}{!}{%
\begin{tabular}{lcccccccccc}
\toprule
Approaches  & Types & Phase & Tokens & $K$ & $\mathcal{E}$($\downarrow$) & $\mathrm{U}$ ($\uparrow$)  & PSNR($\uparrow$) & SSIM($\uparrow$) & LPIPS ($\downarrow$) & r-FID($\downarrow$) \\
\midrule
Vanilla VQ & VQ & Substitution & 512 & 65536 &  0.269 & 0.1\% & 23.39 & 64.9 & 0.176 & 16.2\\
EMA VQ & VQ & Substitution & 512 & 65536 & 0.131 & 31.9\% & 27.87 & 77.4 & 0.071 &  3.17\\
Online VQ & VQ & Substitution & 512 & 65536 & 0.180 & 11.6\% & 26.82 & 74.4 & 0.093 & 4.44\\
Wasserstein VQ & VQ & Substitution & 512 & 65536 & \textbf{\underline{0.111}} & \underline{99.3\%} & 28.32 & \textbf{\underline{78.7}} & \textbf{\underline{0.064}} & \underline{2.63}\\
MMD VQ & VQ & Substitution & 512 & 65536 & \textbf{\underline{0.111}} & 99.1\% & \textbf{\underline{28.33}} & 78.5 & \textbf{\underline{0.064}} &  2.66\\
\midrule
Vanilla VP2 & PQ & Substitution & 512 & 65536 & 0.168 & 11.6\% & 26.67 & 75.1 & 0.086 &  3.78\\
EMA VP2 & PQ & Substitution & 512 & 65536 & \underline{0.115} & \textbf{\underline{100\%}} & \underline{28.25} & \underline{78.5} & \underline{0.065} &  3.03\\
Online VP2 & PQ & Substitution & 512 & 65536 & 0.119 & 99.9\% & 28.14 & 78.3 & 0.066 &  3.01\\
Wasserstein VP2 & PQ & Substitution & 512 & 65536 & 0.123 & 99.9\% & 27.82 & 77.7 & 0.068 &  \underline{2.81}\\
MMD VP2 & PQ & Substitution & 512 & 65536 & 0.120 & \textbf{\underline{100\%}} & 27.99 & 78.1 & 0.069  & 2.92\\
\midrule
FSQ & SQ & Substitution & 512 & 65536 & 0.147 & 99.9\% & \underline{27.58} & 75.9 & 0.075 &  4.64\\
LFQ & SQ & Substitution & 512 & 65536 & 0.211 & 2.9\% & 25.40 & 71.1 & 0.115 &  6.56\\
BSQ & SQ & Substitution & 512 & 65536 & \underline{0.133} & \textbf{\underline{100\%}} & 27.28 & \underline{77.0} & \underline{0.074} &  \textbf{\underline{2.58}}\\
\midrule
\midrule
Vanilla VQ & VQ & Adaptation & 512 & 65536 & 0.269 & 0.1\% &  22.95 & 62.4 & 0.159 & 5.13 \\
EMA VQ & VQ & Adaptation & 512 & 65536 & 0.131 & 31.9\%  & 27.53 & 75.6 & 0.069 & 2.34 \\
Online VQ & VQ & Adaptation & 512 & 65536 &  0.180 & 11.6\%  & 26.30 & 72.0 & 0.090 & 3.46 \\
Wasserstein VQ & VQ & Adaptation & 512 & 65536 &  \textbf{\underline{0.111}} & \underline{99.3\%} & \textbf{\underline{27.98}} & 76.6 & 0.064 & \textbf{\underline{1.73}} \\
MMD VQ & VQ & Adaptation & 512 & 65536 & \textbf{\underline{0.111}} & 99.1\% & \textbf{\underline{27.98}} & \textbf{\underline{76.7}} & \textbf{\underline{0.063}} & 1.76 \\
\midrule
Vanilla VP2 & PQ & Adaptation & 512 & 65536 & 0.168 & 11.6\%  & 26.37 & 72.8 & 0.083 & 2.36 \\
EMA VP2 & PQ & Adaptation & 512 & 65536 & \underline{0.115} & \textbf{\underline{100\%}}  & \underline{27.78} & \underline{76.0} & 0.068 & \underline{1.96} \\
Online VP2 & PQ & Adaptation & 512 & 65536 & 0.119 & 99.9\%  & 27.67 & 75.3 & 0.068 & 2.11 \\
Wasserstein VP2 & PQ & Adaptation & 512 & 65536 &  0.123 & 99.9\%  & 27.67 & 75.9 & 0.068 & 2.02 \\
MMD VP2 & PQ & Adaptation & 512 & 65536 & 0.120 & \textbf{\underline{100\%}} & 27.68 & 75.8 & \underline{0.066} & 2.07 \\
\midrule
FSQ & SQ & Adaptation & 512 & 65536 & 0.147 & 99.9\%  & \underline{27.16} & 73.6 & \underline{0.072} & \underline{2.24} \\
LFQ & SQ & Adaptation & 512 & 65536 & 0.211 & 2.9\%  & 24.88 & 68.1 & 0.111 & 3.35 \\
BSQ & SQ & Adaptation & 512 & 65536 & \underline{0.133} & \textbf{\underline{100\%}} & 26.93 & \underline{74.9} & 0.074 & 2.47 \\
\bottomrule
\end{tabular}
}
\vspace{-6pt}
\end{table*}

\section{Generalization Across Datasets}
\label{appendix:generalization datasets}

To examine whether the main empirical findings generalize across data distributions, we further evaluate VQ, PQ, and SQ on FFHQ~\citep{Karras2018ASG} and CelebA-HQ~\citep{Karras2017ProgressiveGO}. We follow the same controlled latent-space protocol described in Appendix~\ref{appendix:experimental details}, with $T=512$ and $K=65{,}536$ for all methods. Tables~\ref{tab:var-ffhq} and~\ref{tab:var-celebahq} report the results.

\paragraph{Generalization of Quantization Comparisons.}
The results are consistent with the main experimental findings reported in Table~\ref{tab:vq pq sq}. On FFHQ, the best VQ, PQ, and SQ methods achieve distortions of $0.131$, $0.136$, and $0.155$, respectively; after decoder adaptation, their best r-FID values are $0.85$, $1.05$, and $1.54$. On CelebA-HQ, the corresponding best distortions are $0.111$, $0.115$, and $0.133$, while the best r-FID values are $1.73$, $1.96$, and $2.24$. These results show that the advantage of modern VQ methods under controlled feature statistics and matched nominal rate persists across datasets.

\paragraph{Generalization of Distortion--Reconstruction Correlation.}
We further compute Spearman's rank correlations between r-FID and quantization distortion $\mathcal{E}$ or codebook utilization $\mathrm{U}$. On FFHQ, distortion is strongly correlated with r-FID ($\rho=0.979$, $p=5.49\times10^{-9}$), whereas utilization shows a much weaker correlation ($\rho=-0.492$, $p=0.088$). CelebA-HQ exhibits the same pattern, with $\rho=0.944$ ($p=1.29\times10^{-6}$) for distortion and $\rho=-0.559$ ($p=0.047$) for utilization. Together with the ImageNet-1K results in the main text, these findings consistently show that distortion is more strongly associated with reconstruction fidelity than codebook utilization, supporting distortion minimization as a more reliable criterion for intrinsic quantization effectiveness.


\section{Controlled Quantization Comparison in Pixel Space}
\label{appendix:pixel space}
To further examine whether our conclusions depend on the particular latent representation induced by a pretrained encoder, we conduct an additional controlled comparison directly from image pixels on CelebA-HQ~\citep{Karras2017ProgressiveGO}. This experiment removes the dependence on a pretrained latent feature distribution, thereby reducing a potential source of confounding in quantizer comparison. As described in Appendix~\ref{appendix:experimental details}, all methods share the same PixelUnshuffle operation and convolutional projectors, and are evaluated with the same token count $T=4096$ and code-space cardinality $K=65{,}536$. Thus, the compared methods differ primarily in their quantization modules while operating under the same representation pipeline and nominal rate. Table~\ref{tab:pixel-celebahq} reports the results.

\paragraph{Quantization and Reconstruction Performance.}
The pixel-space results are consistent with the latent-space findings. MMD VQ achieves the lowest quantization distortion, $\mathcal{E}=0.0021$, compared with $0.0023$ for the best PQ variant and $0.0032$ for the best SQ variant. It also achieves the best r-FID of $3.64$, compared with $3.77$ for the best PQ method and $4.54$ for the best SQ method. More broadly, modern VQ variants achieve strong reconstruction performance despite PQ and SQ methods often attaining higher codebook utilization. These results indicate that the relative advantage of VQ is not specific to the pretrained latent space used in the main experiments.

\paragraph{Distortion--Reconstruction Correlation.}
We further compute Spearman's rank correlations between r-FID and quantization distortion $\mathcal{E}$ or codebook utilization $\mathrm{U}$. Distortion exhibits a strong positive correlation with r-FID ($\rho=0.947$, $p=2.91\times10^{-6}$), whereas codebook utilization shows a substantially weaker and statistically insignificant negative correlation ($\rho=-0.413$, $p=0.182$). This closely mirrors the latent-space results and provides additional evidence that distortion is more strongly associated with reconstruction fidelity than codebook utilization. Overall, the pixel-space experiments support the same conclusions as the latent-space analysis while reducing dependence on a particular pretrained feature space.

\noindent
\begin{table*}[!t]
\centering
\caption{\small{Quantization and reconstruction performance of VQ, PQ, and SQ methods on CelebA-HQ under the controlled pixel-space setting with matched token count ($T=4096$) and code-space cardinality ($K=65{,}536$). Within each quantization type, the best results are \underline{underlined}, while the overall best result for each metric is shown in \textbf{bold}. Results for Online VQ are omitted because the experiment could not be completed due to out-of-memory errors.}}
\vspace*{-5pt}
\label{tab:pixel-celebahq}
\resizebox{\textwidth}{!}{%
\begin{tabular}{lccccccccc}
\toprule
Approaches  & Types & Tokens & $K$ & $\mathcal{E}$($\downarrow$) & $\mathrm{U}$ ($\uparrow$)  & PSNR($\uparrow$) & SSIM($\uparrow$) & LPIPS ($\downarrow$) & r-FID($\downarrow$) \\
\midrule
Vanilla VQ & VQ  & 4096 & 65536 &  0.0055 & 0.2\% & 30.99 & 86.1 & 0.085 & 8.45\\
EMA VQ & VQ  & 4096 & 65536 & 0.0033 & 4.2\%  & 33.33 & 89.4 & 0.057 &  6.51\\
Online VQ & VQ & 4096 & 65536 & - & - & - & - & - & -  \\
Wasserstein VQ & VQ & 4096 & 65536 & 0.0022 & 96.0\% & 35.32 & \underline{92.9} & 0.026 & 3.75\\
MMD VQ & VQ  & 4096 & 65536 & \textbf{\underline{0.0021}} & \underline{98.5\%} & \textbf{\underline{35.37}} & \underline{92.9} & \textbf{\underline{0.023}} & \textbf{\underline{3.64}}\\
\midrule
Vanilla VP2 & PQ  & 4096 & 65536 & 0.0024 & \textbf{\underline{100\%}} & 34.71 & \textbf{\underline{93.0}} & \underline{0.025} &  3.93\\
EMA VP2 & PQ & 4096 & 65536 & 0.0027 & \textbf{\underline{100\%}} & 34.12 & 90.7 & 0.057 & 6.37\\
Online VP2 & PQ & 4096 & 65536 & \underline{0.0023} & \textbf{\underline{100\%}} & 34.96 & 92.7 & 0.027 & 4.00\\
Wasserstein VP2 & PQ & 4096 & 65536 & 0.0024  & 99.2\% & 34.88  & 92.5  & 0.027 & 4.21\\
MMD VP2 & PQ & 4096 & 65536 & \underline{0.0023} & \textbf{\underline{100\%}} & \underline{35.15} & 92.9 & \underline{0.025} & \underline{3.77}\\
\midrule
FSQ & SQ & 4096 & 65536 & \underline{0.0032} & \textbf{\underline{100\%}} & \underline{33.40} & \underline{91.8} & \underline{0.031} & \underline{4.54}\\
LFQ & SQ & 4096 & 65536 & 0.0050 & 10.1\% & 31.27 & 87.2 & 0.071 & 7.99\\
BSQ & SQ & 4096 & 65536 & 0.0041 & 99.8\% & 32.15 & 90.1 & 0.039 & 4.75\\
\bottomrule
\end{tabular}
}
\vspace{-15pt}
\end{table*}

\end{appendices}

\clearpage

\end{document}